\documentclass{article}

\PassOptionsToPackage{numbers, compress}{natbib}

     \usepackage[final]{neurips_2020}

\usepackage{microtype}
\usepackage[utf8]{inputenc} 
\usepackage[T1]{fontenc}    
\usepackage{nicefrac}
\usepackage{lipsum}
\usepackage{graphicx}
\usepackage{wrapfig}
\usepackage{subfigure}
\usepackage{amsfonts}
\usepackage{commath}
\usepackage{amsmath}
\usepackage{dirtytalk}
\usepackage{amsthm}
\usepackage{mathtools}
\usepackage[shortlabels]{enumitem}
\newcommand\independent{\protect\mathpalette{\protect\independenT}{\perp}}
\usepackage{natbib}
\def\independenT#1#2{\mathrel{\rlap{$#1#2$}\mkern2mu{#1#2}}}

\DeclareMathOperator*{\argmin}{arg\,min}
\DeclareMathOperator*{\Image}{Im}
\newtheorem{theorem}{Theorem}[section]
\newtheorem{lemma}[theorem]{Lemma}

\theoremstyle{definition}
\newtheorem{definition}[theorem]{Definition}
\newtheorem{innercustomproof}{}
\newenvironment{customproof}[1]
{\renewcommand\theinnercustomproof{#1}\innercustomproof}
{\endinnercustomproof}
\usepackage{tikz-cd}
\usepackage{booktabs} 
\usepackage{hyperref}
\allowdisplaybreaks

\title{A Measure-Theoretic Approach to Kernel Conditional Mean Embeddings}

\author{Junhyung Park \\
	MPI for Intelligent Systems, T\"ubingen \\
	\texttt{junhyung.park@tuebingen.mpg.de} \\
	\And
	Krikamol Muandet \\
	MPI for Intelligent Systems, T\"ubingen \\
	\texttt{krikamol@tuebingen.mpg.de} \\
}
\begin{document}

\maketitle
\begin{abstract}
We present an operator-free, measure-theoretic approach to the conditional mean embedding (CME) as a random variable taking values in a reproducing kernel Hilbert space. While the kernel mean embedding of unconditional distributions has been defined rigorously, the existing operator-based approach of the conditional version depends on stringent assumptions that hinder its analysis. We overcome this limitation via a measure-theoretic treatment of CMEs. We derive a natural regression interpretation to obtain empirical estimates, and provide a thorough theoretical analysis thereof, including universal consistency. As natural by-products, we obtain the conditional analogues of the maximum mean discrepancy and Hilbert-Schmidt independence criterion, and demonstrate their behaviour via simulations. 
\end{abstract}

\section{Introduction}\label{Sintro}
The idea of embedding probability distributions into a reproducing kernel Hilbert space (RKHS), a space associated to a positive definite kernel, has received a lot of attention in the past decades \citep{berlinet2004reproducing,smola2007hilbert}, and has found a wealth of successful applications, such as independence testing \citep{gretton2008kernel}, two-sample testing \citep{gretton2012kernel}, learning on distributions \citep{muandet2012learning,lopez2015towards,szabo2016learning}, goodness-of-fit testing \citep{chwialkowski2016kernel,liu2016kernelised} and probabilistic programming \citep{scholkopf2015computing,simongabriel2016consistent}, among others -- see review \citep{muandet2017kernel}. It extends the idea of kernelising linear methods by embedding data points into high- (and often infinite-)dimensional RKHSs, which has been applied, for example, in ridge regression, spectral clustering, support vector machines and principal component analysis among others \citep{scholkopf2001learning,hofmann2008kernel,steinwart2008support}. 

Conditional distributions can also be embedded into RKHSs in a similar manner \citep{song2013kernel},\citep[Chapter 4]{muandet2017kernel}. Compared to unconditional distributions, conditional distributions can represent more complicated relations between random variables, and so conditional mean embeddings (CMEs) have the potential to unlock the arsenal of kernel mean embeddings to a wider setting. Indeed, CMEs have been applied successfully to dynamical systems \citep{song2009hilbert}, inference on graphical models \citep{song2010nonparametric}, probabilistic inference via kernel sum and product rules \citep{song2013kernel}, reinforcement learning \citep{grunewalder2012modelling,nishiyama2012hilbert}, kernelising the Bayes rule and applying it to nonparametric state-space models \citep{fukumizu2013kernel} and causal inference \citep{mitrovic2018causal} to name a few. 

Despite such progress, the current prevalent definition of the CME based on composing cross-covariance operators \citep{song2009hilbert} relies on some stringent assumptions, which are often violated and hinder its analysis. \citet{klebanov2020rigorous} recently attempted to clarify and weaken some of these assumptions, but strong and hard-to-verify conditions still persist. \citet{grunewalder2012conditional} provided a regression interpretation, but here, only the existence of the CME is shown, without an explicit expression. The main contribution of this paper is to remove these stringent assumptions using a novel measure-theoretic approach to the CME. This approach requires drastically weaker assumptions, and comes in an explicit expression. We believe this will enable a more principled analysis of its theoretical properties, and open doors to new application areas. We derive an empirical estimate based on vector-valued regression along with in-depth theoretical analysis, including universal consistency. In particular, we relax the assumption of \cite{grunewalder2012conditional} to allow for infinite-dimensional RKHSs. 

As natural by-products, we obtain quantities that are extensions of the maximum mean discrepancy (MMD) and the Hilbert-Schmidt independence criterion (HSIC) to the conditional setting, which we call the \textit{maximum conditional mean discrepancy} (MCMD) and the \textit{Hilbert-Schmidt conditional independence criterion} (HSCIC). We demonstrate their properties through simulation experiments. 

All proofs can be found in Appendix \ref{Sproofs}. 

\section{Preliminaries}\label{Spreliminaries}
We take \((\Omega,\mathcal{F},P)\) as the underlying probability space. Let \((\mathcal{X},\mathfrak{X})\), \((\mathcal{Y},\mathfrak{Y})\) and \((\mathcal{Z},\mathfrak{Z})\) be separable measurable spaces, and let \(X:\Omega\rightarrow\mathcal{X}\), \(Y:\Omega\rightarrow\mathcal{Y}\) and \(Z:\Omega\rightarrow\mathcal{Z}\) be random variables with distributions \(P_X\), \(P_Y\) and \(P_Z\). We will use \(Z\) as the conditioning variable throughout. 

\subsection{Positive definite kernels and RKHS embeddings}\label{SSrkhs}
Let \(\mathcal{H}_\mathcal{X}\) be a vector space of \(\mathcal{X}\rightarrow\mathbb{R}\) functions, endowed with a Hilbert space structure via an inner product \(\langle\cdot,\cdot\rangle_{\mathcal{H}_\mathcal{X}}\). A symmetric function \(k_\mathcal{X}:\mathcal{X}\times\mathcal{X}\rightarrow\mathbb{R}\) is a \textit{reproducing kernel} of \(\mathcal{H}_\mathcal{X}\) if and only if: 1. \(\forall x\in\mathcal{X}\), \(k_\mathcal{X}(x,\cdot)\in\mathcal{H}_\mathcal{X}\); 2. \(\forall x\in\mathcal{X}\) and \(\forall f\in\mathcal{H}_\mathcal{X}\), \(f(x)=\langle f,k_\mathcal{X}(x,\cdot)\rangle_{\mathcal{H}_\mathcal{X}}\). A space \(\mathcal{H}_\mathcal{X}\) which possesses a reproducing kernel is called a \textit{reproducing kernel Hilbert space} (RKHS) \citep{berlinet2004reproducing}. Throughout this paper, we assume that all RKHSs are \textit{separable}. This is not a restrictive assumption, since it is satisfied if, for example, \(k_\mathcal{X}\) is a continuous kernel \citep[p.130, Lemma 4.33]{steinwart2008support} (for further details, please see \cite{owhadi2017separability}). Given a distribution \(P_X\) on \(\mathcal{X}\), assuming the integrability condition 
\begin{equation}\label{Eintegrability}
\int_\mathcal{X}\sqrt{k_\mathcal{X}(x,x)}dP_X(x)<\infty,
\end{equation}
we define the \textit{kernel mean embedding} \(\mu_{P_X}(\cdot)=\int_\mathcal{X}k_\mathcal{X}(x,\cdot)dP_X(x)\) of \(P_X\), where the integral is a \textit{Bochner integral} \citep[p.15, Def. 35]{dinculeanu2000vector}. We will later show a conditional analogue of the following lemma (for completeness, a proof is provided in Appendix \ref{Sproofs}). 
\begin{lemma}[\cite{smola2007hilbert}]\label{Linterchange}
	For each \(f\in\mathcal{H}_\mathcal{X}\), \(\int_\mathcal{X}f(x)dP_X(x)=\langle f,\mu_{P_X}\rangle_{\mathcal{H}_\mathcal{X}}\). 
\end{lemma}
Next, suppose \(\mathcal{H}_\mathcal{Y}\) is an RKHS of functions on \(\mathcal{Y}\) with kernel \(k_\mathcal{Y}\), and consider the \textit{tensor product RKHS} \(\mathcal{H}_\mathcal{X}\otimes\mathcal{H}_\mathcal{Y}\) (see \citep[pp.47-48]{weidmann1980linear} for a definition of tensor product Hilbert spaces). 
\begin{theorem}[{\citep[p.31, Theorem 13]{berlinet2004reproducing}}]\label{Tberlinet13}
	The tensor product \(\mathcal{H}_\mathcal{X}\otimes\mathcal{H}_\mathcal{Y}\) is generated by the functions \(f\otimes g:\mathcal{X}\times\mathcal{Y}\rightarrow\mathbb{R}\), with \(f\in\mathcal{H}_\mathcal{X}\) and \(g\in\mathcal{H}_\mathcal{Y}\) defined by \((f\otimes g)(x,y)=f(x)g(y)\). Moreover, \(\mathcal{H}_\mathcal{X}\otimes\mathcal{H}_\mathcal{Y}\) is an RKHS of functions on \(\mathcal{X}\times\mathcal{Y}\) with kernel
	\((k_\mathcal{X}\otimes k_\mathcal{Y})((x_1,y_1),(x_2,y_2))=k_\mathcal{X}(x_1,x_2)k_\mathcal{Y}(y_1,y_2).\)
\end{theorem}
Now let us impose a slightly stronger integrability condition:
\begin{equation}\label{Estrongerintegrability}
\mathbb{E}_X[k_\mathcal{X}(X,X)]<\infty,\quad\mathbb{E}_Y[k_\mathcal{Y}(Y,Y)]<\infty.
\end{equation}
This ensures that \(k_\mathcal{X}(X,\cdot)\otimes k_\mathcal{Y}(Y,\cdot)\) is Bochner \(P_{XY}\)-integrable, and so \(\mu_{P_{XY}}\vcentcolon=\mathbb{E}_{XY}[k_\mathcal{X}(X,\cdot)\otimes k_\mathcal{Y}(Y,\cdot)]\in\mathcal{H}_\mathcal{X}\otimes\mathcal{H}_\mathcal{Y}\). The next lemma is analogous to Lemma \ref{Linterchange}:
\begin{lemma}[{\citep[Theorem 1]{fukumizu2004dimensionality}}]\label{Ltensorinterchange}
	For \(f\in\mathcal{H}_\mathcal{X}\), \(g\in\mathcal{H}_\mathcal{Y}\), \(\langle f\otimes g,\mu_{P_{\mathit{XY}}}\rangle_{\mathcal{H}_\mathcal{X}\otimes\mathcal{H}_\mathcal{Y}}=\mathbb{E}_{\mathit{XY}}[f(X)g(Y)]\).
\end{lemma}
As a consequence, for any pair \(f\in\mathcal{H}_\mathcal{X}\) and \(g\in\mathcal{H}_\mathcal{Y}\), we have \(\langle f\otimes g,\mu_{P_{\mathit{XY}}}-\mu_{P_X}\otimes\mu_{P_Y}\rangle_{\mathcal{H}_\mathcal{X}\otimes\mathcal{H}_\mathcal{Y}}=\text{Cov}_{\mathit{XY}}[f(X),g(Y)]\). There exists an isometric isomorphism \(T:\mathcal{H}_\mathcal{X}\otimes\mathcal{H}_\mathcal{Y}\rightarrow\text{HS}(\mathcal{H}_\mathcal{X},\mathcal{H}_\mathcal{Y})\), where \(\text{HS}(\mathcal{H}_\mathcal{X},\mathcal{H}_\mathcal{Y})\) is the space of Hilbert-Schmidt operators from \(\mathcal{H}_\mathcal{X}\) to \(\mathcal{H}_\mathcal{Y}\) (Lemma \ref{Lisomorphism}). The (centred) \textit{cross-covariance operator} is defined as \(\mathcal{C}_{\mathit{YX}}\vcentcolon=T(\mu_{P_{\mathit{XY}}}-\mu_{P_X}\otimes\mu_{P_Y})\) \citep[Theorem 1]{fukumizu2004dimensionality}. The object \(T(\mu_{P_{XY}})\) is referred to as the \textit{uncentred cross-covariance operator} in the literature \citep[Section 3.2]{song2010hilbert}. 

The notion of \textit{characteristic kernels} is essential, since it tells us that the associated RKHSs are rich enough to enable us to distinguish different distributions from their embeddings.
\begin{definition}[\cite{fukumizu2008kernel}]\label{Dcharacteristic}
	A positive definite kernel \(k_\mathcal{X}\) is \textit{characteristic} to a set \(\mathcal{P}\) of probability measures defined on \(\mathcal{X}\) if the map \(\mathcal{P}\rightarrow\mathcal{H}_\mathcal{X}:P_X\mapsto\mu_{P_X}\) is injective. 
\end{definition}
\citet{sriperumbudur2010hilbert} discusses various characterisations of characteristic kernels and show that the well-known Gaussian and Laplacian kernels are characteristic. We then have a metric on \(\mathcal{P}\) via \(\lVert\mu_{P_X}-\mu_{P_{X'}}\rVert_{\mathcal{H}_\mathcal{X}}\) for \(P_X,P_{X'}\in\mathcal{P}\), which is the definition of the MMD \citep{gretton2007kernel}. Furthermore, the HSIC is defined as the Hilbert-Schmidt norm of \(\mathcal{C}_{\mathit{YX}}\), or equivalently, \(\lVert\mu_{P_{\mathit{XY}}}-\mu_{P_X}\otimes\mu_{P_Y}\rVert_{\mathcal{H}_\mathcal{X}\otimes\mathcal{H}_\mathcal{Y}}\) \citep{gretton2005measuring}. If \(k_\mathcal{X}\otimes k_\mathcal{Y}\) is characteristic, then \(\text{HSIC}=0\) if and only if \(X\independent Y\). 

\subsection{Conditioning}\label{Sconditioning}
We briefly review the concept of conditioning in measure-theoretic probability theory, with Banach space-valued random variables. We consider a sub-\(\sigma\)-algebra \(\mathcal{E}\) of \(\mathcal{F}\) and a Banach space \(\mathcal{H}\). 
\begin{definition}[Conditional Expectation, {\citep[p.45, Definition 38]{dinculeanu2000vector}}]\label{Dconditionalexpectation}
	Suppose \(H\) is a Bochner \(P\)-integrable, \(\mathcal{H}\)-valued random variable. Then the \textit{conditional expectation} of \(H\) given \(\mathcal{E}\) is any \(\mathcal{E}\)-measurable, Bochner \(P\)-integrable, \(\mathcal{H}\)-valued random variable \(H'\) such that \(\int_AHdP=\int_AH'dP\) \(\forall A\in\mathcal{E}\). Any \(H'\) satisfying this condition is a \textit{version} of \(\mathbb{E}[H\mid\mathcal{E}]\). We write \(\mathbb{E}[H\mid Z]\) to mean \(\mathbb{E}[H\mid\sigma(Z)]\), where \(\sigma(Z)\) is the sub-\(\sigma\)-algebra of \(\mathcal{F}\) generated by the random variable \(Z\). 
\end{definition}
The (almost sure) uniqueness of the conditional expectation is shown in \citep[p.44, Proposition 37]{dinculeanu2000vector}, and the existence in \citep[pp.45-46, Theorems 39 and 50]{dinculeanu2000vector}. 
\begin{definition}[{\citep[p.149]{cinlar2011probability}}]\label{Dconditionalprobability}
	The \textit{conditional probability} of \(A\in\mathcal{F}\) given \(\mathcal{E}\) is \(P(A\mid\mathcal{E})=\mathbb{E}[\mathbf{1}_A\mid\mathcal{E}]\).
\end{definition}
Note that, in the unconditional case, the expectation is defined as the integral with respect to the measure, but in the conditional case, the expectation is defined \textit{first}, and the measure is \textit{defined} as the expectation of the indicator function. For this definition to be useful, we require an additional property, called \textit{regular version}. We first define the \textit{transition probability kernel}\footnote{Here, the term \say{kernel} must not be confused with the kernel associated to RKHSs.}. 
\begin{definition}[{\citep[p.37,40]{cinlar2011probability}}]\label{Dtransitionprobabilitykernel}
	Let \((\Omega_i,\mathcal{F}_i)\), \(i=1,2\) be measurable spaces. A mapping \(K:\Omega_1\times\mathcal{F}_2\rightarrow[0,\infty]\) is a \textit{transition kernel} from \((\Omega_1,\mathcal{F}_1)\) to \((\Omega_2,\mathcal{F}_2)\) if (i) \(\forall B\in\mathcal{F}_2\), \(\omega\mapsto K(\omega,B)\) is \(\mathcal{F}_1\)-measurable; (ii) \(\forall\omega\in\Omega_1\), \(B\mapsto K(\omega,B)\) is a measure on \((\Omega_2,\mathcal{F}_2)\).	If \(K(\omega,\Omega_2)=1\) \(\forall\omega\in\Omega_1\), \(K\) is said to be a \textit{transition probability kernel}.
\end{definition}
\begin{definition}[{\citep[p.150, Definition 2.4]{cinlar2011probability}}]\label{Dregularversion}
	For each \(A\in\mathcal{F}\), let \(Q(A)\) be a version of \(P(A|\mathcal{E})=\mathbb{E}[\mathbf{1}_A|\mathcal{E}]\). Then \(Q:(\omega,A)\mapsto Q_\omega(A)\) is said to be a \textit{regular version} of the conditional probability measure \(P(\cdot\mid\mathcal{E})\) if \(Q\) is a transition probability kernel from \((\Omega,\mathcal{E})\) to \((\Omega,\mathcal{F})\). 
\end{definition}
The following theorem, proved in Appendix \ref{Sproofs}, is the reason why a regular version is important. It means that, roughly speaking, the conditional expectation is indeed obtained by integration with respect to the conditional measure. 
\begin{theorem}[Adapted from {\citep[p.150, Proposition 2.5]{cinlar2011probability}}]\label{Tregularversion}
	Suppose that \(P(\cdot\mid\mathcal{E})\) admits a regular version \(Q\). Then \(QH:\Omega\rightarrow\mathcal{H}\) with \(\omega\mapsto Q_\omega H=\int_\Omega H(\omega')Q_\omega(d\omega')\)
	is a version of \(\mathbb{E}[H\mid\mathcal{E}]\) for every Bochner \(P\)-integrable \(H\). 
\end{theorem}

\subsection{Vector-valued RKHS regression}\label{SSfunctionregression}
In this subsection, we introduce the theory of vector-valued RKHS regression, based on operator-valued kernels. Let \(\mathcal{H}\) be a Hilbert space, which will be the output space of regression. 
\begin{definition}[{\citep[Definition 1]{carmeli2006vector}}]\label{Drkhsvectorvalued}
	An \textit{\(\mathcal{H}\)-valued RKHS} on \(\mathcal{Z}\) is a Hilbert space \(\mathcal{G}\) such that 1. the elements of \(\mathcal{G}\) are functions \(\mathcal{Z}\rightarrow\mathcal{H}\); 2. \(\forall z\in\mathcal{Z}\), \(\exists C_z>0\) such that \(\lVert F(z)\rVert_\mathcal{H}\leq C_z\lVert F\rVert_\mathcal{G}\) \(\forall F\in\mathcal{G}\).
\end{definition}
Next, we let \(\mathcal{L}(\mathcal{H})\) denote the Banach space of bounded linear operators from \(\mathcal{H}\) into itself. 
\begin{definition}[{\citep[Definition 2]{carmeli2006vector}}]\label{Dpositivedefinitevectorvalued}
	A \textit{\(\mathcal{H}\)-kernel of positive type} on \(\mathcal{Z}\times\mathcal{Z}\) is a map \(\Gamma:\mathcal{Z}\times\mathcal{Z}\rightarrow\mathcal{L}(\mathcal{H})\) such that \(\forall N\in\mathbb{N}\), \(\forall z_1,...,z_N\in\mathcal{Z}\) and \(\forall c_1,...,c_N\in\mathbb{R}\), \(\sum^N_{i,j=1}c_ic_j\langle\Gamma(z_j,z_i)h,h\rangle_\mathcal{H}\geq0\) \(\forall h\in\mathcal{H}\).
\end{definition}
Analogously to the scalar case, it can be shown that any \(\mathcal{H}\)-valued RKHS \(\mathcal{G}\) possesses a \textit{reproducing kernel}, which is an \(\mathcal{H}\)-kernel of positive type \(\Gamma\) satisfying, for any \(z,z'\in\mathcal{Z}\), \(h,h'\in\mathcal{H}\) and \(F\in\mathcal{G}\), \(\langle F(z),h\rangle_\mathcal{H}=\langle F,\Gamma(\cdot,z)h\rangle_\mathcal{G}\) and \(\langle h,\Gamma(z,z')(h')\rangle_\mathcal{H}=\langle\Gamma(\cdot,z)(h),\Gamma(\cdot,z')(h')\rangle_\mathcal{G}\). 

Now suppose we want to perform regression with input space \(\mathcal{Z}\) and output space \(\mathcal{H}\), by minimising
\begin{equation}\label{Emiccheli4.1}
\frac{1}{n}\sum_{j=1}^n\lVert h_j-F(z_j)\rVert_\mathcal{H}^2+\lambda\lVert F\rVert_\mathcal{G}^2,
\end{equation}
where \(\lambda>0\) is a regularisation parameter and \(\{(z_j,h_j):j=1,...,n\}\subseteq\mathcal{Z}\times\mathcal{H}\). There is a corresponding representer theorem (here, \(\delta_{jl}\) is the Kronecker delta):
\begin{theorem}[{\citep[Theorem 4.1]{micchelli2005learning}}]\label{Tmiccheli4.1}
	If \(\hat{F}\) minimises (\ref{Emiccheli4.1}) in \(\mathcal{G}\), it is unique and has the form \(\hat{F}=\sum_{j=1}^n\Gamma(\cdot,z_j)(u_j)\) where the coefficients \(\{u_j:j=1,...,n\}\subseteq\mathcal{H}\) are the unique solution of the linear equations
	\(\sum_{l=1}^n(\Gamma(z_j,z_l)+n\lambda\delta_{jl})(u_l)=h_j,j=1,...,n\).
\end{theorem}

\section{Conditional mean embedding}\label{Scme}
We are now ready to introduce a formal definition of the conditional mean embedding of \(X\) given \(Z\).
\begin{definition}\label{Dcme}
	Assuming \(X\) satisfies the integrability condition (\ref{Eintegrability}), we define the \textit{conditional mean embedding} of \(X\) given \(Z\) as \(\mu_{P_{X|Z}}\vcentcolon=\mathbb{E}_{X|Z}[k_\mathcal{X}(X,\cdot)\mid Z]\). 
\end{definition}
This is a direct extension of the unconditional kernel mean embedding, \(\mu_{P_X}=\mathbb{E}_X[k_\mathcal{X}(X,\cdot)]\), but instead of being a fixed element in \(\mathcal{H}_\mathcal{X}\), \(\mu_{P_{X|Z}}\) is a \(Z\)-measurable random variable taking values in \(\mathcal{H}_\mathcal{X}\) (see Definition \ref{Dconditionalexpectation}). Also, for any function \(f:\mathcal{X}\rightarrow\mathbb{R}\), \(\mathbb{E}_{X|Z}[f(X)\mid Z]\) is a real-valued \(Z\)-measurable random variable. The following lemma is analogous to Lemma \ref{Linterchange}. 
\begin{lemma}\label{Lconditionalinterchange}
	For any \(f\in\mathcal{H}_\mathcal{X}\), \(\mathbb{E}_{X|Z}[f(X)\mid Z]=\langle f,\mu_{P_{X|Z}}\rangle_{\mathcal{H}_\mathcal{X}}\) almost surely.
\end{lemma}
Next, assuming \(X\) and \(Y\) satisfy (\ref{Estrongerintegrability}), we define \(\mu_{P_{XY|Z}}\vcentcolon=\mathbb{E}_{XY|Z}[k_\mathcal{X}(X,\cdot)\otimes k_\mathcal{Y}(Y,\cdot)\mid Z]\), a \(Z\)-measurable, \(\mathcal{H}_\mathcal{X}\otimes\mathcal{H}_\mathcal{Y}\)-valued random variable. We have the following analogy of Lemma \ref{Ltensorinterchange}: 
\begin{lemma}\label{Lconditionaltensorinterchange}
	For any pair \(f\in\mathcal{H}_\mathcal{X}\) and \(g\in\mathcal{H}_\mathcal{Y}\), \(\mathbb{E}_{\mathit{XY}|Z}[f(X)g(Y)\mid Z]=\langle f\otimes g,\mu_{P_{\mathit{XY}|Z}}\rangle_{\mathcal{H}_\mathcal{X}\otimes\mathcal{H}_\mathcal{Y}}\) almost surely. 
\end{lemma}
By Lemmas \ref{Lconditionalinterchange} and \ref{Lconditionaltensorinterchange}, for any pair \(f\in\mathcal{H}_\mathcal{X}\) and \(g\in\mathcal{H}_\mathcal{Y}\), 
\begin{alignat*}{2}
\langle f\otimes g,\mu_{P_{\mathit{XY}|Z}}-\mu_{P_{X|Z}}\otimes\mu_{P_{Y|Z}}&\rangle_{\mathcal{H}_\mathcal{X}\otimes\mathcal{H}_\mathcal{Y}}=\text{Cov}_{\mathit{XY}|Z}(f(X),g(Y)\mid Z)\\
&=\mathbb{E}_{\mathit{XY}|Z}[f(X)g(Y)\mid Z]-\mathbb{E}_{X|Z}[f(X)\mid Z]\mathbb{E}_{Y|Z}[g(Y)\mid Z]
\end{alignat*}
almost surely. Hence, we define the \textit{conditional cross-covariance operator} as \(\mathcal{C}_{YX|Z}\vcentcolon=T(\mu_{P_{XY|Z}}-\mu_{P_{X|Z}}\otimes\mu_{P_{Y|Z}})\) (see Section \ref{SSrkhs} for the definition of \(T\)). 

\subsection{Comparison with existing definitions}\label{SScomparison}
As previously mentioned, the idea of CMEs and conditional cross-covariance operators is not a novel one, yet our development of the theory above differs significantly from the existing works. In this subsection, we review the previous approaches and compare them to ours. 

The prevalent definition of CMEs in the literature is the following. We first need to endow the conditioning space \(\mathcal{Z}\) with a scalar kernel, say \(k_\mathcal{Z}:\mathcal{Z}\times\mathcal{Z}\rightarrow\mathbb{R}\), with corresponding RKHS \(\mathcal{H}_\mathcal{Z}\). 
\begin{definition}[{\citep[Definition 3]{song2009hilbert}}]\label{Dconditionalmeanembedding}
	The conditional mean embedding of the conditional distribution \(P(X\mid Z)\) is the operator \(\mathcal{U}_{X|Z}:\mathcal{H}_\mathcal{Z}\rightarrow\mathcal{H}_\mathcal{X}\) defined by \(\mathcal{U}_{X|Z}=\mathcal{C}_{\mathit{XZ}}\mathcal{C}_{\mathit{ZZ}}^{-1}\), where \(\mathcal{C}_{\mathit{XZ}}\) and \(\mathcal{C}_{\mathit{ZZ}}\) are unconditional (cross-)covariance operators as defined in Section \ref{SSrkhs}. 
\end{definition}
As noted by \citep{song2009hilbert}, the motivation for this comes from \citep[Theorem 2]{fukumizu2004dimensionality}, which states that for any \(f\in\mathcal{H}_\mathcal{X}\), if \(\mathbb{E}_{X|Z}[f(X)\mid Z=\cdot]\in\mathcal{H}_\mathcal{Z}\), then  \(\mathcal{C}_{\mathit{ZZ}}\mathbb{E}_{X|Z}[f(X)\mid Z=\cdot]=\mathcal{C}_{\mathit{ZX}}f\). This relation can be used to prove the following theorem, which is analogous to Lemma \ref{Lconditionalinterchange}. 
\begin{theorem}[{\citep[Theorem 4]{song2009hilbert}}]\label{Tsong4}
	For \(f\in\mathcal{H}_\mathcal{X}\), assuming \(\mathbb{E}_{X|Z}[f(X)\mid Z=\cdot]\in\mathcal{H}_\mathcal{Z}\), \(\mathcal{U}_{X|Z}\) satisfies: 1. \(\mu_{X|z}\vcentcolon=\mathbb{E}_{X|z}[k_\mathcal{X}(X,\cdot)\mid Z=z]=\mathcal{U}_{X|Z}k_\mathcal{Z}(z,\cdot)\); 2. \(\mathbb{E}_{X|z}[f(X)\mid Z=z]=\langle f,\mu_{X|z}\rangle_{\mathcal{H}_\mathcal{X}}\).
\end{theorem}
Now we highlight the key differences between this approach and ours. Firstly, this approach requires the endowment of a kernel \(k_\mathcal{Z}\) on the conditioning space \(\mathcal{Z}\), and defines the CME as an \textit{operator} from \(\mathcal{H}_\mathcal{Z}\) to \(\mathcal{H}_\mathcal{X}\). By contrast, Definition \ref{Dcme} did not consider any kernel or function on \(\mathcal{Z}\), and defined the CME as a \textit{Bochner conditional expectation} given \(\sigma(Z)\). We argue that it is more natural not to endow the \textit{conditioning space} with a kernel before the estimation stage. Secondly, the operator-based approach assumes that \(\mathbb{E}_{X|Z}[f(X)|Z=\cdot]\), as a function in \(z\), lives in \(\mathcal{H}_\mathcal{Z}\). This is a severe restriction; it is stated in \citep{song2009hilbert} that this assumption, while true for finite domains with characteristic kernels, is not necessarily true for continuous domains, and \citep{fukumizu2013kernel} gives a simple counterexample using the Gaussian kernel. Lastly, it also assumes that \(\mathcal{C}_{\mathit{ZZ}}^{-1}\) exists, which is another unrealistic assumption. \citep{fukumizu2013kernel} mentions that this assumption is too strong in many situations, and gives a counterexample using the Gaussian kernel. The most common remedy is to resort to the regularised version \(\mathcal{C}_{\mathit{XZ}}(\mathcal{C}_{\mathit{ZZ}}+\lambda I)^{-1}\) and treat it as an approximation of \(\mathcal{U}_{X|Z}\). These assumptions have been clarified and slightly weakened in \citep{klebanov2020rigorous}, but strong and hard-to-verify conditions persist. In contrast, Definition \ref{Dcme} extend the notions of kernel mean embedding, expectation operator and cross-covariance operator to the conditional setting simply by using the formal definition of conditional expectations (Definition \ref{Dconditionalexpectation}), and the subsequent result in Lemma \ref{Lconditionalinterchange}, analogous to \citep[Theorem 4]{song2009hilbert}, does not rely on any assumptions. 

A regression interpretation is given in \cite{grunewalder2012conditional}, by showing the \textit{existence}, for each \(z\in\mathcal{Z}\), of \(\mu(z)\in\mathcal{H}_\mathcal{X}\) that satisfies \(\mathbb{E}[h(X)\mid Z=z]=\langle h,\mu(z)\rangle_{\mathcal{H}_\mathcal{X}}\). However, no explicit expression for \(\mu(z)\) is provided. In contrast, our definition provides an explicit expression \(\mu_{P_{X|Z}}=\mathbb{E}_{X|Z}[k_\mathcal{X}(X,\cdot)\mid Z]\). 

In \citep[Section A.2]{fukumizu2004dimensionality}, the conditional cross-covariance operator is defined, but in a significantly different way. It is defined as \(\Sigma_{\mathit{YX}|Z}\vcentcolon=\mathcal{C}_{\mathit{YX}}-\mathcal{C}_{\mathit{YZ}}\tilde{\mathcal{C}}_{\mathit{ZZ}}^{-1}\mathcal{C}_{\mathit{ZX}}\), where \(\tilde{\mathcal{C}}_{\mathit{ZZ}}^{-1}\) is the right inverse of \(\mathcal{C}_{\mathit{ZZ}}\) on \((\textnormal{Ker}\mathcal{C}_{\mathit{ZZ}})^\perp\). This has the property that, for all \(f\in\mathcal{H}_\mathcal{X}\) and \(g\in\mathcal{H}_\mathcal{Y}\), \(\langle g,\Sigma_{\mathit{YX}|Z}f\rangle_{\mathcal{H}_\mathcal{Y}}=\mathbb{E}_Z[\textnormal{Cov}_{\mathit{XY}|Z}(f(X),g(Y)\mid Z)]\). Note that this is different to our relation stated after Lemma \ref{Lconditionaltensorinterchange}; the conditional covariance is integrated out over \(\mathcal{Z}\). In fact, this difference is explicitly noted by \cite{song2009hilbert}. 

\section{Empirical estimates}\label{Sempirical}
In this section, we discuss how we can obtain empirical estimates of \(\mu_{P_{X|Z}}=\mathbb{E}_{X|Z}[k_\mathcal{X}(X,\cdot)\mid Z]\). 
\begin{theorem}\label{Tcinlar4.4generalised}
	Denote the Borel \(\sigma\)-algebra of \(\mathcal{H}_\mathcal{X}\) by \(\mathcal{B}(\mathcal{H}_\mathcal{X})\). Then we can write \(\mu_{P_{X|Z}}=F_{P_{X|Z}}\circ Z\), where \(F_{P_{X|Z}}:\mathcal{Z}\rightarrow\mathcal{H}_\mathcal{X}\) is some deterministic function, measurable with respect to \(\mathfrak{Z}\) and \(\mathcal{B}(\mathcal{H}_\mathcal{X})\). 
\end{theorem}
Hence, estimating \(\mu_{P_{X|Z}}\) boils down to estimating the function \(F_{P_{X|Z}}\), which is exactly the setting for vector-valued regression (Section \ref{SSfunctionregression}) with input space \(\mathcal{Z}\) and output space \(\mathcal{H}_\mathcal{X}\). In contrast to \cite{grunewalder2012conditional}, where regression is motivated by applying the Riesz representation theorem conditioned on each value of \(z\in\mathcal{Z}\), we derive the CME as an explicit function of \(Z\), which we argue is a more principled way to motivate regression. Moreover, for continuous \(Z\), the event \(Z=z\) has measure 0, so it is not measure-theoretically rigorous to apply the Riesz representation theorem conditioned on \(Z=z\).  

The natural optimisation problem is to minimise the loss \(\mathcal{E}_{X|Z}(F)\vcentcolon=\mathbb{E}_Z[\lVert F_{P_{X|Z}}(Z)-F(Z)\rVert^2_{\mathcal{H}_\mathcal{X}}]\) among all \(F\in\mathcal{G}_{\mathcal{X}\mathcal{Z}}\), where \(\mathcal{G}_{\mathcal{X}\mathcal{Z}}\) is a vector-valued RKHS of functions \(\mathcal{Z}\rightarrow\mathcal{H}_\mathcal{X}\). For simplicity, we endow \(\mathcal{G}_{\mathcal{X}\mathcal{Z}}\) with a kernel \(l_{\mathcal{X}\mathcal{Z}}(z,z')=k_\mathcal{Z}(z,z')\text{Id}\), where \(k_\mathcal{Z}(\cdot,\cdot)\) is a scalar kernel on \(\mathcal{Z}\).\footnote{\(\mathcal{E}_{X|Z}\) is not the only loss function, nor is \(l_{\mathcal{X}\mathcal{Z}}\) the only kernel, that we can use for this problem. \citet{kadri2016operator} discuss various operator-valued kernels that can be used (albeit without closed-form solutions) and \citet{laforgue2020duality} discuss other loss functions that can be used for more robust estimates. We view this flexibility to facilitate other loss and kernel functions in the regression set-up, although not explored in depth in this work, as a significant advantage over the previous approaches.}

We cannot minimise \(\mathcal{E}_{X|Z}\) directly, since we do not observe samples from \(\mu_{P_{X|Z}}\), but only the pairs \((x_i,z_i)\) from \((X,Z)\). We bound this with a surrogate loss \(\tilde{\mathcal{E}}_{X|Z}\) that has a sample-based version:
\begin{alignat*}{2}
\mathcal{E}_{X|Z}(F)&=\mathbb{E}_Z[\lVert\mathbb{E}_{X|Z}[k_\mathcal{X}(X,\cdot)-F(Z)\mid Z]\rVert^2_{\mathcal{H}_\mathcal{X}}]\leq\mathbb{E}_Z\mathbb{E}_{X|Z}[\lVert k_\mathcal{X}(X,\cdot)-F(Z)\rVert^2_{\mathcal{H}_\mathcal{X}}\mid Z]\\
&=\mathbb{E}_{X,Z}[\lVert k_\mathcal{X}(X,\cdot)-F(Z)\rVert^2_{\mathcal{H}_\mathcal{X}}]=\vcentcolon\tilde{\mathcal{E}}_{X|Z}(F),
\end{alignat*}
where we used generalised conditional Jensen's inequality (see Appendix \ref{SgeneralisedJensen}, or \cite{perlman1974jensen}). Section \ref{SSuniversalityconsistency} discusses the meaning of this surrogate loss. We replace the surrogate population loss with a regularised empirical loss based on samples \(\{(x_i,z_i)\}_{i=1}^n\) from the joint distribution \(P_{\mathit{XZ}}\): \(\hat{\mathcal{E}}_{X|Z,n,\lambda}(F)\vcentcolon=\frac{1}{n}\sum^n_{i=1}\lVert k_\mathcal{X}(x_i,\cdot)-F(z_i)\rVert^2_{\mathcal{H}_\mathcal{X}}+\lambda\lVert F\rVert^2_{\mathcal{G}_{\mathcal{X}\mathcal{Z}}}\), where \(\lambda>0\) is a regularisation parameter. We see that this loss functional is exactly in the form of (\ref{Emiccheli4.1}). Therefore, by Theorem \ref{Tmiccheli4.1}, the minimiser \(\hat{F}_{P_{X|Z},n,\lambda}\) of \(\hat{\mathcal{E}}_{X|Z,n,\lambda}\) is \(\hat{F}_{P_{X|Z},n,\lambda}(\cdot)=\mathbf{k}_Z^T(\cdot)\mathbf{f}\), where \(\mathbf{k}_Z(\cdot)\vcentcolon=(k_Z(z_1,\cdot),...,k_Z(z_n,\cdot))^T\), \(\mathbf{f}\vcentcolon=(f_1,...,f_n)^T\) and the coefficients \(f_i\in\mathcal{H}_\mathcal{X}\) are the unique solutions of the linear equations \((\mathbf{K}_Z+n\lambda\mathbf{I})\mathbf{f}=\mathbf{k}_X\), where \([\mathbf{K}_Z]_{ij}\vcentcolon=k_\mathcal{Z}(z_i,z_j)\), \(\mathbf{k}_X\vcentcolon=(k_\mathcal{X}(x_1,\cdot),...,k_\mathcal{X}(x_n,\cdot))^T\) and \(\mathbf{I}\) is the \(n\times n\) identity matrix. Hence, the coefficients are \(\mathbf{f}=\mathbf{W}\mathbf{k}_X\), where \(\mathbf{W}=(\mathbf{K}_Z+n\lambda\mathbf{I})^{-1}\).
Finally, substituting this into the expression for \(\hat{F}_{P_{X|Z},n,\lambda}(\cdot)\), we have
\begin{equation}\label{Eempirical}
\hat{F}_{P_{X|Z},n,\lambda}(\cdot)=\mathbf{k}_Z^T(\cdot)\mathbf{W}\mathbf{k}_X\in\mathcal{G}_{\mathcal{X}\mathcal{Z}}.
\end{equation}

\subsection{Surrogate loss, universality and consistency}\label{SSuniversalityconsistency}
In this subsection, we investigate the meaning and consequences of using the surrogate loss \(\tilde{\mathcal{E}}_{X|Z}\) instead of the original \(\mathcal{E}_{X|Z}\), as well as the universal consistency property of our learning algorithm. 

Denote by \(L^2(\mathcal{Z},P_Z;\mathcal{H}_\mathcal{X})\) the Banach space of (equivalence classes of) measurable functions \(F:\mathcal{Z}\rightarrow\mathcal{H}_\mathcal{X}\) such that \(\lVert F(\cdot)\rVert_{\mathcal{H}_\mathcal{X}}^2\) is \(P_Z\)-integrable, with norm \(\lVert F\rVert_2=(\int_\mathcal{Z}\lVert F(z)\rVert^2_{\mathcal{H}_\mathcal{X}}dP_Z(z))^{\frac{1}{2}}\). We can note that the true function \(F_{P_{X|Z}}\) belongs to \(L^2(\mathcal{Z},P_Z;\mathcal{H}_\mathcal{X})\), because Theorem \ref{Tcinlar4.4generalised} tells us that \(F_{P_{X|Z}}\) is indeed measurable, and by Theorem \ref{TgeneralisedconditionalJensen} and (\ref{Estrongerintegrability}), \(\int_\mathcal{Z}\lVert F_{P_{X|Z}}(z)\rVert_{\mathcal{H}_\mathcal{X}}^2dP_Z(z)=\mathbb{E}_Z[\lVert\mathbb{E}_{X|Z}[k_\mathcal{X}(X,\cdot)\mid Z]\rVert_{\mathcal{H}_\mathcal{X}}^2]\leq\mathbb{E}_Z[\mathbb{E}_{X|Z}[\lVert k_\mathcal{X}(X,\cdot)\rVert_{\mathcal{H}_\mathcal{X}}^2\mid Z]]=\mathbb{E}_X[\lVert k_\mathcal{X}(X,\cdot)\rVert_{\mathcal{H}_\mathcal{X}}^2]<\infty\). The true function \(F_{P_{X|Z}}\) is the unique minimiser in \(L^2(\mathcal{Z},P_Z;\mathcal{H}_\mathcal{X})\) of both \(\mathcal{E}_{X|Z}\) and \(\tilde{\mathcal{E}}_{X|Z}\):
\begin{theorem}\label{Tgrunewalder3.1}
	\(F_{P_{X|Z}}\) minimises both \(\tilde{\mathcal{E}}_{X|Z}\) and \(\mathcal{E}_{X|Z}\) in \(L^2(\mathcal{Z},P_Z;\mathcal{H}_\mathcal{X})\). Moreover, it is almost surely equal to any other minimiser of the loss functionals. 
\end{theorem}
Note the difference in the statement of Theorem \ref{Tgrunewalder3.1} from \citep[Theorem 3.1]{grunewalder2012conditional}, which only considers the minimisation of the loss functionals in \(\mathcal{G}_{\mathcal{X}\mathcal{Z}}\), whereas we consider the larger space \(L^2(\mathcal{Z},P_Z;\mathcal{H}_\mathcal{X})\).
Next, we discuss the concepts of \textit{universal kernels} and \textit{universal consistency}. 
\begin{definition}[{\cite[Definition 2]{carmeli2010vector}}]\label{DC0universality}
	A kernel \(l_{\mathcal{X}\mathcal{Z}}:\mathcal{Z}\times\mathcal{Z}\rightarrow\mathcal{L}(\mathcal{H}_\mathcal{X})\) with RKHS \(\mathcal{G}_{\mathcal{X}\mathcal{Z}}\) is \(\mathcal{C}_0\) if \(\mathcal{G}_{\mathcal{X}\mathcal{Z}}\) is a subspace of \(\mathcal{C}_0(\mathcal{Z},\mathcal{H}_\mathcal{X})\), the space of continuous functions \(\mathcal{Z}\rightarrow\mathcal{H}_\mathcal{X}\) vanishing at infinity. The kernel \(l_{\mathcal{X}\mathcal{Z}}\) is \textit{\(\mathcal{C}_0\)-universal} if is is \(\mathcal{C}_0\) and \(\mathcal{G}_{\mathcal{X}\mathcal{Z}}\) is dense in \(L^2(\mathcal{Z},P_Z;\mathcal{H}_\mathcal{X})\) for any measure \(P_Z\) on \(\mathcal{Z}\). 
\end{definition}
\citet[Example 14]{carmeli2010vector} shows that \(l_{\mathcal{X}\mathcal{Z}}=k_\mathcal{Z}(\cdot,\cdot)\text{Id}\) is \(\mathcal{C}_0\)-universal if \(k_\mathcal{Z}\) is a universal scalar kernel, which in turn is guaranteed if \(k_\mathcal{Z}\) is Gaussian or Laplacian, for example \citep{steinwart2001influence}. 
\begin{figure*}[t]
	\begin{center}
		\centerline{\includegraphics[scale=0.7]{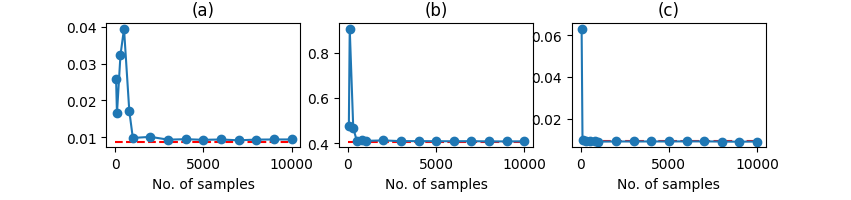}}
		\caption{Solid blue and dashed red lines represent  \(\tilde{\mathcal{E}}_{X|Z}(\hat{F}_{P_{X|Z},n,\lambda_n})\) and \(\tilde{\mathcal{E}}_{X|Z}(F_{P_{X|Z}})\) respectively.}
		\label{Fconsistency}
	\end{center}
\end{figure*}
The consistency result with optimal rate \(\mathcal{O}_p(\frac{\log{n}}{n})\) in \citep[Corollaries 4.1, 4.2]{grunewalder2012conditional} is based on \cite{caponnetto2006optimal}, and assumes, along with some distributional assumptions, that \(\mathcal{H}_\mathcal{X}\) is finite-dimensional, which is not true for many common choices of \(k_\mathcal{X}\) (see Appendix \ref{Sgeneralisation} for more details). In \citep[Theorem 6]{song2009hilbert}, \citep[Theorem 1]{song2010nonparametric} and \citep[Theorem 1.3.2]{fukumizu2015nonparametric}, consistency is also shown under various assumptions, with rates at best \(\mathcal{O}_p(n^{-\frac{1}{4}})\). In Theorem \ref{Tconsistency}, we prove universal consistency without any distributional assumptions, and in Theorem \ref{Trate}, we show that a convergence rate of \(\mathcal{O}_p(n^{-1/4})\) can be achieved with a simple smoothness assumption that \(F_{P_{X|Z}}\in\mathcal{G}_{\mathcal{X}\mathcal{Z}}\) (sometimes referred to as the \textit{well-specified case}; see \cite{szabo2016learning}). In particular, both results relax the finite-dimensionality assumption on \(\mathcal{H}_\mathcal{X}\) of \cite{grunewalder2012conditional}. 
\begin{theorem}\label{Tconsistency}
	Suppose that \(k_\mathcal{X}\) and \(k_\mathcal{Z}\) are bounded kernels, i.e. there are \(B_\mathcal{Z},B_\mathcal{X}>0\) with \(\sup_{z\in\mathcal{Z}}k_\mathcal{Z}(z,z)\leq B^2_\mathcal{Z}\), \(\sup_{x\in\mathcal{X}}k_\mathcal{X}(x,x)\leq B^2_\mathcal{X}\), and that the operator-valued kernel \(l_{\mathcal{X}\mathcal{Z}}\) is \(\mathcal{C}_0\)-universal. Let the regularisation parameter \(\lambda_n\) decay to 0 at a slower rate than \(\mathcal{O}(n^{-1/2})\). Then the learning algorithm that yields \(\hat{F}_{P_{X|Z},n,\lambda_n}\) is universally consistent, i.e. for any joint distribution \(P_{XZ}\), \(\epsilon>0\) and \(\delta>0\), \(P_{\mathit{XZ}}(\tilde{\mathcal{E}}_{X|Z}(\hat{F}_{P_{X|Z},n,\lambda_n})-\tilde{\mathcal{E}}_{X|Z}(F_{P_{X|Z}})>\epsilon)<\delta\) for sufficiently large \(n\). 
\end{theorem}
Figure  experimentally verifies universal consistency under three noise levels. We use the distributions \(Z\sim\mathcal{N}(0, 1)\), (a) \(X=e^{-\frac{1}{2}Z^2}\sin(2Z)+N_a\), \(N_a\sim0.3\mathcal{N}(0,1)\); (b) \(X=e^{-\frac{1}{2}Z^2}\sin(2Z)+N_b\), \(N_b\sim3\mathcal{N}(0,1)\); (c) \(X=Z+N_a\), with regularisation \(\lambda_n=10^{-7}n^{-\frac{1}{4}}\). 
\begin{theorem}\label{Trate}
	Assume further that \(F_{P_{X|Z}}\in\mathcal{G}_{\mathcal{X}\mathcal{Z}}\). Then with probability at least \(1-\delta\), 
	\begin{equation*}
	\begin{split}
	&\tilde{\mathcal{E}}_{X|Z}(\hat{F}_{P_{X|Z},n,\lambda_n})-\tilde{\mathcal{E}}_{X|Z}(F_{P_{X|Z}})\leq\lambda_n\left\lVert F_{P_{X|Z}}\right\rVert^2_{\mathcal{G}_{\mathcal{X}\mathcal{Z}}}\\
	&+\frac{2\ln\left(\frac{4}{\delta}\right)}{3n\lambda_n}\left(1+\sqrt{1+\frac{18n}{\ln\left(\frac{4}{\delta}\right)}}\right)\left(\left(B_\mathcal{Z}\left\lVert F_{P_{X|Z}}\right\rVert_{\mathcal{G}_{\mathcal{X}\mathcal{Z}}}+B_\mathcal{X}\right)^2\lambda_n+B_\mathcal{X}^2\left(B_\mathcal{Z}+\sqrt{\lambda_n}\right)^2\right)
	\end{split}
	\end{equation*}
\end{theorem}
In particular, if \(\lambda_n=\mathcal{O}(n^{-1/4})\), then \(\tilde{\mathcal{E}}_{X|Z}(\hat{F}_{P_{X|Z},n,\lambda_n})-\tilde{\mathcal{E}}_{X|Z}(F_{P_{X|Z}})=\mathcal{O}_p(n^{-1/4})\). The boundedness assumption is satisfied with many commonly used kernels, such as the Gaussian and Laplacian, and hence is not a restrictive condition. Note that some smoothness assumption on \(F_{P_{X|Z}}\) or other distributional assumptions are necessary to achieve universal convergence rates, otherwise the rates can be arbitrarily slow -- for more discussion, see e.g. \citep[p.56]{vapnik1998nature}, \citep[p.114, Theorem 7.2]{devroye1996probabilistic} or \citep[p.32, Theorem 3.1]{gyorfi2006distribution}. It is likely that better (and even optimal) rates can be achieved with further assumptions (see e.g. \cite{caponnetto2006optimal,steinwart2009optimal,blanchard2018optimal} for results with real or finite-dimensional output spaces), but we leave further investigation of learning rates with infinite-dimensional output spaces as future work. 

Theorem \ref{Tconsistency} is stated with respect to the surrogate loss \(\tilde{\mathcal{E}}_{X|Z}\), not the original loss \(\mathcal{E}_{X|Z}\). Let us now investigate its implications with respect to the original loss. Write \(\eta=\tilde{\mathcal{E}}_{X|Z}(F_{P_{X|Z}})\). Since \(\tilde{\mathcal{E}}_{X|Z}(\hat{F}_{P_{X|Z},n,\lambda_n})\geq\mathcal{E}_{X|Z}(\hat{F}_{P_{X|Z},n,\lambda_n})\), a consequence of Theorem \ref{Tconsistency} is that \(\lim_{n\rightarrow\infty}P_{\mathit{XZ}}(\mathcal{E}_{X|Z}(\hat{F}_{P_{X|Z},n,\lambda_n})>\epsilon+\eta)\leq\lim_{n\rightarrow\infty}P_{\mathit{XZ}}(\tilde{\mathcal{E}}_{X|Z}(\hat{F}_{P_{X|Z},n,\lambda_n})-\eta>\epsilon)=0\) for any \(\epsilon>0\). This shows that, in the limit as \(n\rightarrow\infty\), the loss \(\mathcal{E}_{X|Z}(\hat{F}_{P_{X|Z},n,\lambda_n})\) is at most an arbitrarily small amount larger than \(\eta\) with high probability. 

It remains to investigate what \(\eta\) represents, and how large it is. The law of total expectation gives \(\eta=\mathbb{E}_{X,Z}[\lVert k_\mathcal{X}(X,\cdot)-F_{P_{X|Z}}(Z)\rVert_{\mathcal{H}_\mathcal{X}}^2]=\mathbb{E}_Z[\mathbb{E}_{X|Z}[\lVert k_\mathcal{X}(X,\cdot)-\mathbb{E}_{X|Z}[k_\mathcal{X}(X,\cdot)\mid Z]\rVert_{\mathcal{H}_\mathcal{X}}^2\mid Z]]\). Here, the integrand \(\mathbb{E}_{X|Z}[\lVert k_\mathcal{X}(X,\cdot)-\mathbb{E}_{X|Z}[k_\mathcal{X}(X,\cdot)|\mid Z]\rVert_{\mathcal{H}_\mathcal{X}}^2\mid Z]\) is the \textit{variance} of \(k_\mathcal{X}(X,\cdot)\) given \(Z\) (see \citep[p.24]{bharucha1972random} for the definition of the variance of Banach-space valued random variables), and by integrating over \(\mathcal{Z}\) in the outer integral, \(\eta\) represents the \say{expected variance} of \(k_\mathcal{X}(X,\cdot)\).

Suppose \(X\) is measurable with respect to \(Z\), i.e. \(F_{P_{X|Z}}\) has no noise. Then \(\mathbb{E}_{X|Z}[k_\mathcal{X}(X,\cdot)\mid Z]=k_\mathcal{X}(X,\cdot)\), and consequently, \(\eta=0\). In this case, we have universal consistency in both the surrogate loss \(\tilde{\mathcal{E}}_{X|Z}\) and the original loss \(\mathcal{E}_{X|Z}\). On the other hand, \(\eta\) will be large if information about \(Z\) tells us little about \(X\), and subsequently \(k_\mathcal{X}(X,\cdot)\in\mathcal{H}_\mathcal{X}\). In the extreme case where \(X\) and \(Z\) are independent, we have \(\mathbb{E}_{X|Z}[k_\mathcal{X}(X,\cdot)\mid Z]=\mathbb{E}_X[k_\mathcal{X}(X,\cdot)]\), and \(\eta=\mathbb{E}_X[\lVert k_\mathcal{X}(X,\cdot)-\mathbb{E}_X[k_\mathcal{X}(X,\cdot)]\rVert_{\mathcal{H}_\mathcal{X}}^2]\), which is precisely the variance of \(k_\mathcal{X}(X,\cdot)\) in \(\mathcal{H}_\mathcal{X}\). Hence, \(\eta\) represents the irreducible loss of the true function due to noise in \(X\), and the surrogate loss represents the loss functional taking noise into account, while the original loss measures the deviance from the true conditional expectation. 

\section{Measures of discrepancy between conditional distributions and conditional independence}
In this section, we propose conditional analogues of the maximum mean discrepancy (MMD) and the Hilbert-Schmidt independence criterion (HSIC), to measure, respectively, the discrepancy between conditional distributions and conditional independence.
\subsection{Maximum conditional mean discrepancy}\label{SSdiscrepancy}
Let \(X':\Omega\rightarrow\mathcal{X}\), \(Z':\Omega\rightarrow\mathcal{Z}\) be additional random variables, with \(\int_\mathcal{X}\sqrt{k_\mathcal{X}(x',x')}dP_{X'}(x')<\infty\). Following Theorem \ref{Tcinlar4.4generalised}, we write \(\mu_{P_{X|Z}}=F_{P_{X|Z}}\circ Z\) and \(\mu_{P_{X'|Z'}}=F_{P_{X'|Z'}}\circ Z'\). 
\begin{definition}\label{Dmcmd}
	We define the \textit{maximum conditional mean discrepancy} (MCMD) between \(P_{X|Z}\) and \(P_{X'|Z'}\) to be the function \(\mathcal{Z}\rightarrow\mathbb{R}\) defined by \(M_{P_{X|Z},P_{X'|Z'}}(z)=\lVert F_{P_{X|Z}}(z)-F_{P_{X'|Z'}}(z)\rVert_{\mathcal{H}_\mathcal{X}}\).
\end{definition}
Using \(\{(x_i,z_i)\}_{i=1}^n,\{(x'_j,z'_j)\}_{j=1}^m\) from joint distributions \(P_{\mathit{XZ}},P_{\mathit{X'Z'}}\), we obtain a closed-form, plug-in estimate from (\ref{Eempirical}) for the square of the MCMD function as
\begin{alignat*}{2}
&\hat{M}_{P_{X|Z},P_{X'|Z'}}^2(\cdot)=\lVert\hat{F}_{P_{X|Z},n,\lambda}(\cdot)-\hat{F}_{P_{X'|Z'},m,\lambda'}(\cdot)\rVert_{\mathcal{H}_\mathcal{X}}^2\\
&=\mathbf{k}^T_Z(\cdot)\mathbf{W}_Z\mathbf{K}_X\mathbf{W}_Z^T\mathbf{k}_Z(\cdot)-2\mathbf{k}^T_Z(\cdot)\mathbf{W}_Z\mathbf{K}_{\mathit{XX'}}\mathbf{W}_{Z'}^T\mathbf{k}_{Z'}(\cdot)+\mathbf{k}^T_{Z'}(\cdot)\mathbf{W}_{Z'}\mathbf{K}_{X'}\mathbf{W}_{Z'}^T\mathbf{k}_{Z'}(\cdot),
\end{alignat*}
where \([\mathbf{K}_X]_{ij}=k_\mathcal{X}(x_i,x_j)\), \([\mathbf{K}_{X'}]_{ij}=k_\mathcal{X}(x'_i,x'_j)\), \([\mathbf{K}_{\mathit{XX'}}]_{ij}=k_\mathcal{X}(x_i,x'_j)\), \([\mathbf{K}_{\mathit{Z'}}]_{ij}=k_\mathcal{X}(z'_i,z'_j)\), \(\mathbf{k}_{Z'}(\cdot)=(k_Z(z'_1,\cdot),...,k_Z(z'_m,\cdot))^T\), \(\mathbf{W}_Z=(\mathbf{K}_Z+n\lambda\mathbf{I}_n)^{-1}\) and \(\mathbf{W}_{Z'}=(\mathbf{K}_{Z'}+m\lambda'\mathbf{I}_m)^{-1}\).

The term MMD stems from the equality \(\lVert\mu_{P_X}-\mu_{P_{X'}}\rVert_{\mathcal{H}_\mathcal{X}}=\sup_{f\in\mathcal{B}_\mathcal{X}}\lvert\mathbb{E}_X[f(X)]-\mathbb{E}_{X'}[f(X')]\rvert\) \citep{gretton2007kernel,sriperumbudur2010hilbert}, where \(\mathcal{B}_\mathcal{X}\vcentcolon=\{f\in\mathcal{H}_\mathcal{X}\mid\lVert f\rVert_{\mathcal{H}_\mathcal{X}}\leq1\}\). The supremum is attained by the \textit{witness function}, \(\frac{\mu_{P_X}-\mu_{P_{X'}}}{\lVert\mu_{P_X}-\mu_{P_{X'}}\rVert_{\mathcal{H}_\mathcal{X}}}\) \citep{gretton2012kernel}. Using Lemma \ref{Lconditionalinterchange}, the analogous (almost sure) equality for the MCMD is \(\sup_{f\in\mathcal{B}_\mathcal{X}}\lvert\mathbb{E}_{X|Z}[f(X)\mid Z]-\mathbb{E}_{X'|Z'}[f(X')\mid Z']\rvert=\lVert\mu_{P_{X|Z}}-\mu_{P_{X'|Z'}}\rVert_{\mathcal{H}_\mathcal{X}}\). We define the \textit{conditional witness function} as the \(\mathcal{H}_\mathcal{X}\)-valued random variable \(\frac{\mu_{P_{X|Z}}-\mu_{P_{X'|Z'}}}{\lVert\mu_{P_{X|Z}}-\mu_{P_{X'|Z'}}\rVert_{\mathcal{H}_\mathcal{X}}}\). We can informally think of \(\text{MCMD}_{P_{X|Z},P_{X'|Z'}}(z)\) as \say{MMD between \(P_{X|Z=z}\) and \(P_{X'|Z'=z}\)}. However, we do not have i.i.d. samples from \(P_{X|Z=z}\) and \(P_{X'|Z'=z}\), and hence the estimation cannot be done by U- or V-statistic procedures as done for the MMD. The following theorem says that, with characteristic kernels, the MCMD can indeed act as a discrepancy measure between conditional distributions. 
\begin{theorem}\label{Tmcmd}
	Suppose that \(k_\mathcal{X}\) is characteristic, that \(P_Z\) and \(P_{Z'}\) are absolutely continuous with respect to each other, and that \(P(\cdot\mid Z)\) and \(P(\cdot\mid Z')\) admit regular versions. Then \(M_{P_{X|Z},P_{X'|Z'}}=0\) almost everywhere if and only if, for almost all \(z\in\mathcal{Z}\), \(P_{X|Z=z}(B)=P_{X'|Z'=z}(B)\) for all \(B\in\mathfrak{X}\). 
\end{theorem}
By \cite[p.11 \& p.151, Theorem 2.10]{cinlar2011probability}, we know that the space \((\Omega,\mathcal{F})\) being a Polish space with its Borel \(\sigma\)-algebra is a sufficient condition for \(P(\cdot\mid\mathcal{E})\) to have a regular version for any sub-\(\sigma\)-algebra \(\mathcal{E}\) of \(\mathcal{F}\). Hence, the assumption that \(P(\cdot\mid Z)\) admits a regular version is not a restrictive one. 

The MCMD is reminiscent of the \textit{conditional maximum mean discrepancy} of \citep{ren2016conditional}, defined as the Hilbert-Schmidt norm of the operator \(\mathcal{U}_{X|Z}-\mathcal{U}_{X'|Z}\) (see Definition \ref{Dconditionalmeanembedding}). However, due to previously discussed assumptions, \(\mathcal{U}_{X|Z}\) and \(\mathcal{U}_{X'|Z}\) often do not even exist, and/or do not have the desired properties of Theorem \ref{Tsong4}, so even at population level, \(\mathcal{U}_{X|Z}-\mathcal{U}_{X'|Z}\) is often not an exact measure of discrepancy between conditional distributions, unlike the MCMD. Moreover, \cite{ren2016conditional} only considers the case when the conditioning variable is the same.

\begin{figure*}[t]
	\begin{center}
		\centerline{\includegraphics[scale=0.3]{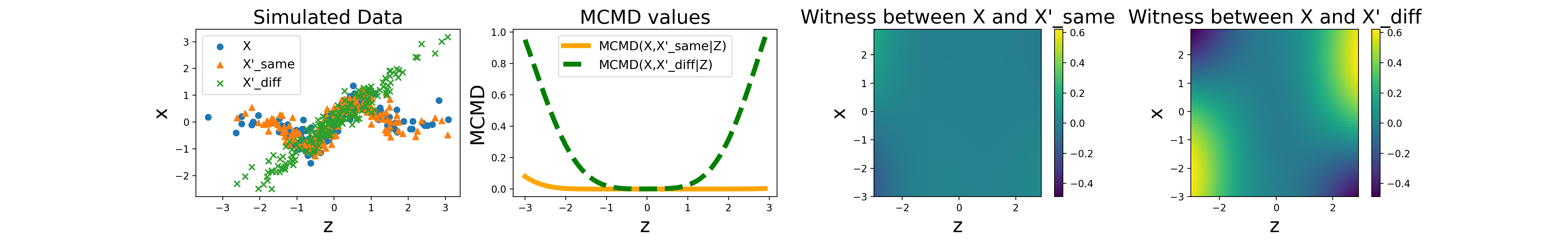}}
		\caption{We see that \(\text{MCMD}(X,X'_\text{same}|Z)\approx0\) \(\forall Z\). Near \(Z=0\), where the dependence on \(Z\) of \(X\) and \(X'_\text{diff}\) are similar, \(\text{MCMD}(X,X'_\text{diff}|Z)\approx0\), whereas away from 0, the dependence on \(Z\) of \(X\) and \(X'_\text{diff}\) are different, and so \(\text{MCMD}(X,X'_\text{diff}|Z)>0\). We also see that the conditional witness function between \(X\) and \(X'_\text{same}\) gives 0 at all values of \(X\) given any value of \(Z\), whereas we have a saddle-like function between \(X\) and \(X'_\text{diff}\), with non-zero functions in \(X\) in the regions of \(Z\) away from 0.}
		\label{Fmcmd}
	\end{center}
\end{figure*}

\subsection{Hilbert-Schmidt conditional independence criterion}\label{SSconditionalindependence}
In this subsection, we introduce a novel criterion of conditional independence.
\begin{definition}\label{Dhscic}
	We define the \textit{Hilbert-Schmidt Conditional Independence Criterion} between \(X\) and \(Y\) given \(Z\) to be \(\text{HSCIC}(X,Y\mid Z)=\lVert\mu_{P_{\mathit{XY}|Z}}-\mu_{P_{X|Z}}\otimes\mu_{P_{Y|Z}}\rVert_{\mathcal{H}_\mathcal{X}\otimes\mathcal{H}_\mathcal{Y}}\).
\end{definition}
We can write \(\text{HSCIC}(X,Y\mid Z)=H_{X,Y|Z}\circ Z\)
for some \(H_{X,Y|Z}:\mathcal{Z}\rightarrow\mathbb{R}\). Given a sample \(\{(x_i,y_i,z_i)\}^n_{i=1}\) from \(P_{XYZ}\), we obtain a plug-in, closed-form estimate of \(H_{X,Y|Z}^2(\cdot)\) as follows:
\begin{alignat*}{2}
\hat{H}^2_{X,Y|Z}(\cdot)&=\mathbf{k}_Z^T(\cdot)\mathbf{W}(\mathbf{K}_X\odot\mathbf{K}_Y)\mathbf{W}^T\mathbf{k}_Z(\cdot)-2\mathbf{k}_Z^T(\cdot)\mathbf{W}((\mathbf{K}_X\mathbf{W}^T\mathbf{k}_Z(\cdot))\odot(\mathbf{K}_Y\mathbf{W}^T\mathbf{k}_Z(\cdot)))\\
&\enspace+(\mathbf{k}_Z^T(\cdot)\mathbf{W}\mathbf{K}_X\mathbf{W}^T\mathbf{k}_Z(\cdot))(\mathbf{k}_Z^T(\cdot)\mathbf{W}\mathbf{K}_Y\mathbf{W}^T\mathbf{k}_Z(\cdot))
\end{alignat*}
where \([\mathbf{K}_Y]_{ij}\vcentcolon=k_\mathcal{Y}(y_i,y_j)\) and \(\odot\) denotes elementwise multiplication of matrices. 

Casting aside measure-theoretic issues arising from conditioning on an event of probability 0, we can conceptually think of the realisation of the HSCIC at each \(z=Z(\omega)\) as \say{the HSIC between \(P_{X|Z=z}\) and \(P_{Y|Z=z}\)}. Again, we do not have multiple samples from each distribution \(P_{X|Z=z}\) and \(P_{Y|Z=z}\), so the estimation cannot be done by U- or V-statistic procedures as done for HSIC. The following theorem shows that HSCIC is a measure of conditional independence. 
\begin{theorem}\label{Thscic}
	Suppose \(k_\mathcal{X}\otimes k_\mathcal{Y}\) is a characteristic kernel\footnote{See \citep{szabo2017characteristic} for a detailed discussion on characteristic tensor product kernels.} on \(\mathcal{X}\times\mathcal{Y}\), and that \(P(\cdot\mid Z)\) admits a regular version. Then \(\textnormal{HSCIC}(X,Y\mid Z)=0\) almost surely if and only if \(X\independent Y\mid Z\). 
\end{theorem}
Concurrent and independent work by \citet{sheng2019distance} proposes a similar criterion with the same nomenclature (HSCIC). However, they omit the discussion of CMEs entirely, and define the HSCIC as the usual HSIC between \(P_{XY|Z=z}\) and \(P_{X|Z=z}P_{Y|Z=z}\), without considerations for conditioning on an event of measure 0. Their focus is more on investigating connections to distance-based measures \citep{wang2015conditional,sejdinovic2013equivalence}. \citet{fukumizu2008kernel} propose \(I^{\mathit{COND}}\), defined as the squared Hilbert-Schmidt norm of the normalised conditional cross-covariance operator \(V_{\ddot{Y}\ddot{X}|Z}\vcentcolon=\mathcal{C}_{\ddot{Y}\ddot{Y}}^{-1/2}\Sigma_{\ddot{Y}\ddot{X}|Z}\mathcal{C}_{\ddot{X}\ddot{X}}^{-1/2}\), where \(\ddot{X}\vcentcolon=(X,Z)\) and \(\ddot{Y}\vcentcolon=(Y,Z)\). As discussed, these operator-based definitions rely on a number of strong assumptions that will often mean that \(V_{\ddot{Y}\ddot{X}|Z}\) does not exist, or it does not satisfy the conditions for it to be used as an exact criterion even at population level. On the other hand, the HSCIC defined as in Definition \ref{Dhscic} is an exact mathematical criterion of conditional independence at population level. Note that \(I^{\mathit{COND}}\) is a single-value criterion, whereas the HSCIC is a random criterion. 

\begin{figure*}[t]
	\begin{center}
		\centerline{\includegraphics[scale=0.28]{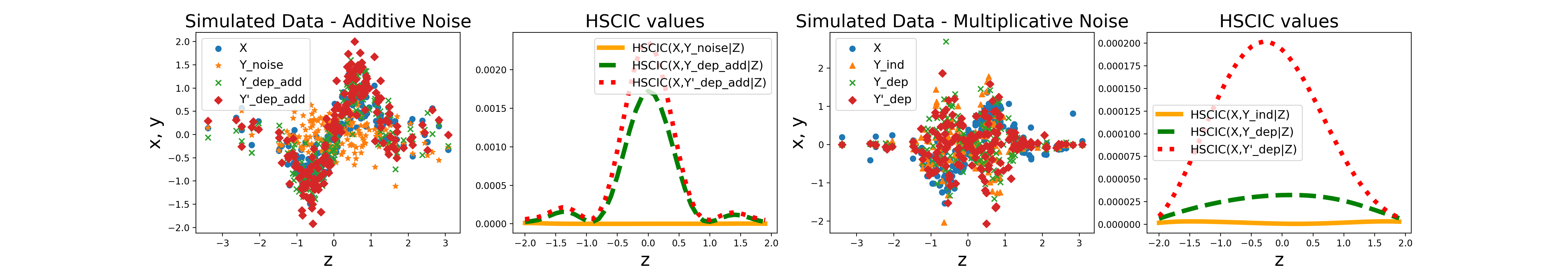}}
		\caption{We see that \(\text{HSCIC}(X,Y_\text{noise}|Z)\approx0\) (left) and \(\text{HSCIC}(X,Y_\text{ind}|Z)\approx0\) (right) for all \(Z\), whereas \(\text{HSCIC}(X,Y_\text{dep\_add}|Z)>0\), \(\text{HSCIC}(X,Y'_\text{dep\_add}|Z)>0\), \(\text{HSCIC}(X,Y_\text{dep}|Z)>0\), \(\text{HSCIC}(X,Y'_\text{dep}|Z)>0\). In particular, the dependence of \(Y'_\text{dep\_add}\) and \(Y'_\text{dep}\) on \(X\) is greater than that of \(Y_\text{dep\_add}\) and \(Y_\text{dep}\), and is represented by larger values of \(\text{HSCIC}(X,Y'_\text{dep\_add}|Z)\) and \(\text{HSCIC}(X,Y'_\text{dep}|Z)\) compared to \(\text{HSCIC}(X,Y_\text{dep}|Z)\) and \(\text{HSCIC}(X,Y_\text{dep\_add}|Z)\). }
		\label{Fhscic}
	\end{center}
\end{figure*}

\subsection{Experiments}\label{Sexperiments}
We carry out simulations to demonstrate the behaviour of the MCMD and HSCIC. In all simulations, we use the Gaussian kernel \(k_\mathcal{X}(x,x')=k_\mathcal{Y}(x,x')=k_\mathcal{Z}(x,x')=e^{-\frac{1}{2}\sigma_X\lVert x-x'\rVert^2_2}\) with hyperparameter \(\sigma_X=0.1\), and regularisation parameter \(\lambda=0.01\). 

In Figure \ref{Fmcmd}, we simulate 500 samples from \(Z,Z'\sim\mathcal{N}(0,1)\), \(X=e^{-0.5Z^2}\sin(2Z)+N_X\), \(X'_\text{same}=e^{-0.5Z'^2}\sin(2Z')+N_X\) and \(X'_\text{diff}=Z'+N_X\), where \(N_X\sim0.3\mathcal{N}(0,1)\) is the (additive) noise variable. The first plot shows simulated data, the second MCMD values against Z, and the heatmaps show the (unnormalised) conditional witness function, whose norm gives the MCMD. 

In Figure \ref{Fhscic}, on the left, we simulate 500 samples from the additive noise model, \(Z\sim\mathcal{N}(0,1)\), \(X=e^{-0.5Z^2}\sin(2Z)+N_X\), \(Y_\text{noise}=N_Y\), \(Y_\text{dep\_add}=e^{-0.5Z^2}\sin(2Z)+N_X+0.2X\) and \(Y'_\text{dep\_add}=e^{-0.5Z^2}\sin(2Z)+N_X+0.4X\), where \(N_X\sim0.3\mathcal{N}(0,1)\) is the (additive) noise variable. On the right, we simulate 500 samples from the multiplicative noise model, \(Z\sim\mathcal{N}(0,1)\), \(X=Y_\text{ind}=e^{-0.5Z^2}\sin(2Z)N_X\), \(Y_\text{dep}=e^{-0.5Z^2}\sin(2Z)N_Y+0.2X\) and \(Y'_\text{dep}=e^{-0.5Z^2}\sin(2Z)N_Y+0.4X\), where \(N_X,N_Y\sim0.3\mathcal{N}(0,1)\) are the (multiplicative) noise variables. 

\section{Conclusion}\label{Sconclusion}
In this paper, we proposed a new approach to kernel conditional mean embeddings, based on Bochner conditional expectation. Compared to the previous operator-based approaches, it does not rely on stringent assumptions that are often violated in common situations. Using this new approach, we discussed how to obtain empirical estimates via natural vector-valued regression, establishing universal consistency under no distributional assumptions and convergence rate of \(\mathcal{O}_p(n^{-1/4})\) in the well-specified case. Finally, we extended the notions of the MMD, witness function and HSIC to the conditional case. We believe that our new approach has the potential to unlock the powerful arsenal of kernel mean embeddings to the conditional setting, in a more convenient and rigorous manner. 

\section*{Broader Impact}
The nature of this work is theoretical, and hence we do not feel it is applicable to discuss its broader societal impact. 

\ack
We thank Mattes Mollenhauer at Freie Universit\"at Berlin for pointing out the missing conditions on the regularization parameter of our initial universal consistency result, and for other fruitful discussions. We also thank anonymous reviewers for pointing out typos, suggesting several improvements and correcting a mistake in the proof of Theorem \ref{Tcinlar4.4generalised}. Finally, we thank Simon Buchholz, Alessandro Ialongo, Heiner Kremer and Jonas K\"ubler at MPI T\"ubingen for helpful feedback on initial drafts. 

The idea behind this paper was conceived, and part of the work done, while JP was a Master's student at the Seminar for Statistics, Department of Mathematics, ETH Z\"urich. JP is extremely grateful to his Master's thesis supervisor, Professor Sara van de Geer, for readily accepting the proposed topic, and her expert guidance throughout the thesis. 

This work was funded by the federal and state governments of Germany through the Max Planck Society (MPG). 

\bibliography{ref}
\bibliographystyle{abbrvnat}
\newpage
\appendix
\section{Generalised Jensen's Inequality}\label{SgeneralisedJensen}
In Section \ref{Sempirical}, we require a version of Jensen's inequality generalised to (possibly) infinite-dimensional vector spaces, because our random variable takes values in \(\mathcal{H}_\mathcal{X}\), and our convex function is \(\lVert\cdot\rVert_{\mathcal{H}_\mathcal{X}}^2:\mathcal{H}_\mathcal{X}\rightarrow\mathbb{R}\). Note that this square norm function is indeed convex, since, for any \(t\in[0,1]\) and any pair \(f,g\in\mathcal{H}_\mathcal{X}\), 
\begin{alignat*}{3}
\lVert tf+(1-t)g\rVert_{\mathcal{H}_\mathcal{X}}^2&\leq(t\lVert f\rVert_{\mathcal{H}_\mathcal{X}}+(1-t)\lVert g\rVert_{\mathcal{H}_\mathcal{X}})^2\qquad&&\text{by the triangle inequality}\\
&\leq t\lVert f\rVert_{\mathcal{H}_\mathcal{X}}^2+(1-t)\lVert g\rVert_{\mathcal{H}_\mathcal{X}}^2,&&\text{by the convexity of }x\mapsto x^2.
\end{alignat*}
The following theorem generalises Jensen's inequality to infinite-dimensional vector spaces. 
\begin{theorem}[Generalised Jensen's Inequality, \citep{perlman1974jensen}, Theorem 3.10]\label{TgeneralisedJensen}
	Suppose \(\mathcal{T}\) is a real Hausdorff locally convex (possibly infinite-dimensional) linear topological space, and let \(C\) be a closed convex subset of \(\mathcal{T}\). Suppose \((\Omega,\mathcal{F},P)\) is a probability space, and \(V:\Omega\rightarrow\mathfrak{T}\) a Pettis-integrable random variable such that \(V(\Omega)\subseteq C\). Let \(f:C\rightarrow[-\infty,\infty)\) be a convex, lower semi-continuous extended-real-valued function such that \(\mathbb{E}_V[f(V)]\) exists. Then 
	\[f(\mathbb{E}_V[V])\leq\mathbb{E}_V[f(V)].\]
\end{theorem}

We will actually apply generalised Jensen's inequality with conditional expectations, so we need the following theorem. 
\begin{theorem}[Generalised Conditional Jensen's Inequality]\label{TgeneralisedconditionalJensen}
	Suppose \(\mathcal{T}\) is a real Hausdorff locally convex (possibly infinite-dimensional) linear topological space, and let \(C\) be a closed convex subset of \(\mathcal{T}\). Suppose \((\Omega,\mathcal{F},P)\) is a probability space, and \(V:\Omega\rightarrow\mathcal{T}\) a Pettis-integrable random variable such that \(V(\Omega)\subseteq C\). Let \(f:C\rightarrow[-\infty,\infty)\) be a convex, lower semi-continuous extended-real-valued function such that \(\mathbb{E}_V[f(V)]\) exists. Suppose \(\mathcal{E}\) is a sub-\(\sigma\)-algebra of \(\mathcal{F}\). Then 
	\[f(\mathbb{E}[V\mid\mathcal{E}])\leq\mathbb{E}[f(V)\mid\mathcal{E}].\]
\end{theorem}
\begin{proof}
	Let \(\mathcal{T}^*\) be the dual space of all real-valued continuous linear functionals on \(\mathcal{T}\). The first part of the proof of \citep[Theorem 3.6]{perlman1974jensen} tells us that, for all \(v\in\mathcal{T}\), we can write
	\[f(v)=\sup\{m(v)\mid m\text{ affine, }m\leq f\text{ on }C\},\]
	where an \textit{affine} function \(m\) on \(\mathcal{T}\) is of the form \(m(v)=v^*(v)+\alpha\) for some \(v^*\in\mathcal{T}^*\) and \(\alpha\in\mathbb{R}\). If we define the subset \(Q\) of \(\mathcal{T}^*\times\mathbb{R}\) as
	\[Q\vcentcolon=\{(v^*,\alpha):v^*\in\mathcal{T}^*,\alpha\in\mathbb{R},v^*(v)+\alpha\leq f(v)\text{ for all }v\in\mathcal{T}\},\]
	then we can rewrite \(f\) as
	\begin{equation}\label{Econvexfunction}
	f(v)=\sup_{(v^*,\alpha)\in Q}\{v^*(v)+\alpha\},\qquad\text{for all }v\in\mathcal{T}.
	\end{equation}	
	
	See that, for any \((v^*,\alpha)\in Q\), we have
	\begin{alignat*}{3}
	\mathbb{E}\left[f(V)\mid\mathcal{E}\right]&\geq\mathbb{E}\left[v^*(V)+\alpha\mid\mathcal{E}\right]&&\text{almost surely, by assumption (*)}\\
	&=\mathbb{E}\left[v^*\left(V\right)\mid\mathcal{E}\right]+\alpha\qquad&&\text{almost surely, by linearity (**).}
	\end{alignat*}
	Here, (*) and (**) use the properties of conditional expectation of vector-valued random variables given in \citep[pp.45-46, Properties 43 and 40 respectively]{dinculeanu2000vector}. 

	We want to show that \(\mathbb{E}\left[v^*(V)\mid\mathcal{E}\right]=v^*\left(\mathbb{E}\left[V\mid\mathcal{E}\right]\right)\) almost surely, and in order to so, we show that the right-hand side is a version of the left-hand side. The right-hand side is clearly \(\mathcal{E}\)-measurable, since we have a linear operator on an \(\mathcal{E}\)-measurable random variable. Moreover, for any \(A\in\mathcal{E}\), 
	\begin{alignat*}{3}
	\int_Av^*\left(\mathbb{E}\left[V\mid\mathcal{E}\right]\right)dP&=v^*\left(\int_A\mathbb{E}\left[V\mid\mathcal{E}\right]dP\right)\qquad&&\text{by \citep[p.403, Proposition E.11]{cohn2013measure}}\\
	&=v^*\left(\int_AVdP\right)&&\text{by the definition of conditional expectation}\\
	&=\int_Av^*\left(V\right)dP&&\text{by \citep[p.403, Proposition E.11]{cohn2013measure}}
	\end{alignat*}
	(here, all the equalities are almost-sure equalities). Hence, by the definition of the conditional expectation, we have that \(\mathbb{E}\left[v^*(V)\mid\mathcal{E}\right]=v^*\left(\mathbb{E}\left[V\mid\mathcal{E}\right]\right)\) almost surely. Going back to our above work, this means that
	\[\mathbb{E}\left[f(V)\mid\mathcal{E}\right]\geq v^*\left(\mathbb{E}\left[V\mid\mathcal{E}\right]\right)+\alpha.\]
	
	Now take the supremum of the right-hand side over \(Q\). Then (\ref{Econvexfunction}) tells us that
	\[\mathbb{E}\left[f(V)\mid\mathcal{E}\right]\geq f\left(\mathbb{E}\left[V\mid\mathcal{E}\right]\right),\]
	as required. 
\end{proof}
In the context of Section \ref{Sempirical}, \(\mathcal{H}_\mathcal{X}\) is real and Hausdorff, and locally convex (because it is a normed space). We take the closed convex subset to be the whole space \(\mathcal{H}_\mathcal{X}\) itself. The function \(\lVert\cdot\rVert_{\mathcal{H}_\mathcal{X}}^2:\mathcal{H}_\mathcal{X}\rightarrow\mathbb{R}\) is convex (as shown above) and continuous, and finally, since Bochner-integrability implies Pettis integrability, all the conditions of Theorem \ref{TgeneralisedconditionalJensen} are satisfied. 

\section{Generalisation Error Bounds}\label{Sgeneralisation}
\citet{caponnetto2006optimal} give an optimal rate of convergence of vector-valued RKHS regression estimators, and its results are quoted by \citet{grunewalder2012conditional} as the state of the art convergence rates, \(O(\frac{\log{n}}{n})\). In particular, this implies that the learning algorithm is consistent. However, the lower rate uses an assumption that the output space is a finite-dimensional Hilbert space \citep[Theorem 2]{caponnetto2006optimal}; and in our case, this will mean that \(\mathcal{H}_\mathcal{X}\) is finite-dimensional. This is not true if, for example, we take \(k_\mathcal{X}\) to be the Gaussian kernel; indeed, this is noted as a limitation by \citet{grunewalder2012conditional}, stating that \say{It is likely that this (finite-dimension) assumption can be weakened, but this requires a deeper analysis}. In this paper, we do not want to restrict our attention to finite-dimensional \(\mathcal{H}_\mathcal{X}\). The upper bound would have been sufficient to guarantee consistency, but an assumption used in the upper bound requires the operator \(l_{XZ,z}:\mathcal{H}_\mathcal{X}\rightarrow\mathcal{G}_{\mathcal{X}\mathcal{Z}}\) defined by
\[l_{XZ,z}(f)(z')=l_{XZ}(z,z')(f)\]
to be Hilbert-Schmidt for all \(z\in\mathcal{Z}\). However, for each \(z\in\mathcal{Z}\), taking any orthonormal basis \(\{\varphi_i\}_{i=1}^\infty\) of \(\mathcal{H}_\mathcal{X}\), we see that
\begin{alignat*}{2}
\sum^\infty_{i=1}\langle l_{XZ,z}(\varphi_i),l_{XZ,z}(\varphi_i)\rangle_{\mathcal{G}_{\mathcal{X}\mathcal{Z}}}
&=\sum^\infty_{i=1}\langle k_\mathcal{Z}(z,\cdot)\varphi_i,k_\mathcal{Z}(z,\cdot)\varphi_i\rangle_{\mathcal{G}_{\mathcal{X}\mathcal{Z}}}\\
&=\sum^\infty_{i=1}\langle k_\mathcal{Z}(z,z)\varphi_i,\varphi_i\rangle_{\mathcal{H}_\mathcal{X}}\\
&=k_\mathcal{Z}(z,z)\sum^\infty_{i=1}1\\
&=\infty,
\end{alignat*}
meaning this assumption is not fulfilled with our choice of kernel either. Hence, results in \cite{caponnetto2006optimal}, used by \cite{grunewalder2012conditional}, are not applicable to guarantee consistency in our context.

\citet{kadri2016operator} address the problem of generalisability of function-valued learning algorithms, using the concept of uniform algorithmic stability \cite{bousquet2002stability}. Let us write
\[\mathcal{D}\vcentcolon=\{(x_1,z_1),...,(x_n,z_n)\}\]
for our training set of size \(n\) drawn i.i.d. from the distribution \(P_{XZ}\), and we denote by \(\mathcal{D}^i=\mathcal{D}\backslash(x_i,z_i)\) the set \(\mathcal{D}\) from which the data point \((x_i,z_i)\) is removed. Further, we denote by \(\hat{F}_{P_{X|Z},\mathcal{D}}=\hat{F}_{P_{X|Z},n,\lambda}\) the estimate produced by our learning algorithm from the dataset \(\mathcal{D}\) by minimising the loss \(\hat{\mathcal{E}}_{X|Z,n,\lambda}(F)=\sum^n_{i=1}\lVert k_\mathcal{X}(x_i,\cdot)-F(z_i)\rVert_{\mathcal{H}_\mathcal{X}}^2+\lambda\lVert F\rVert_{\mathcal{G}_{\mathcal{X}\mathcal{Z}}}^2\)

The assumptions used in this paper, with notations translated to our context, are
\begin{enumerate}[1.]
	\item There exists \(\kappa_1>0\) such that for all \(z\in\mathcal{Z}\), 
	\[\lVert l_{\mathcal{X}\mathcal{Z}}(z,z)\rVert_{\text{op}}=\sup_{f\in\mathcal{H}_\mathcal{X}}\frac{\lVert l_{\mathcal{X}\mathcal{Z}}(z,z)(f)\rVert_{\mathcal{H}_\mathcal{X}}}{\lVert f\rVert_{\mathcal{H}_\mathcal{X}}}\leq\kappa_1^2.\]
	\item The real function \(\mathcal{Z}\times\mathcal{Z}\rightarrow\mathbb{R}\) defined by
	\[(z_1,z_2)\mapsto\langle l_{\mathcal{X}\mathcal{Z}}(z_1,z_2)f_1,f_2\rangle_{\mathcal{H}_\mathcal{X}}\]
	is measurable for all \(f_1,f_2\in\mathcal{H}_\mathcal{X}\). 
	\item The map \((f,F,z)\mapsto\lVert f-F(z)\rVert_{\mathcal{H}_\mathcal{X}}^2\) is \(\tau\)-admissible, i.e. convex with respect to \(F\) and Lipschitz continuous with respect to \(F(z)\), with \(\tau\) as its Lipschitz constant. 
	\item There exists \(\kappa_2>0\) such that for all \((z,f)\in\mathcal{Z}\times\mathcal{H}_\mathcal{X}\) and any training set \(\mathcal{D}\),
	\[\lVert f-\hat{F}_{P_{X|Z},\mathcal{D}}(z)\rVert_{\mathcal{H}_\mathcal{X}}^2\leq\kappa_2.\]
\end{enumerate}
The concept of \textit{uniform stability}, with notations translated to our context, is defined as follows.
\begin{definition}[Uniform algorithmic stability, {\citep[Definition 6]{kadri2016operator}}]\label{Duniformstability}
	For each \(F\in\mathcal{G}_{\mathcal{X}\mathcal{Z}}\), define the function
	\begin{alignat*}{2}
	\mathcal{R}(F):&\mathcal{Z}\times\mathcal{H}_\mathcal{X}\rightarrow\mathbb{R}\\
	&(z,x)\mapsto\lVert k_\mathcal{X}(x,\cdot)-F(z)\rVert_{\mathcal{H}_\mathcal{X}}^2.
	\end{alignat*}
	
	A learning algorithm that calculates the estimate \(\hat{F}_{P_{X|Z},\mathcal{D}}\) from a training set has uniform stability \(\beta\) with respect to the squared loss if the following holds: for all \(n\geq1\), all \(i\in\{1,...,n\}\) and any training set \(\mathcal{D}\) of size \(n\), 
	\[\lVert\mathcal{R}(\hat{F}_{P_{X|Z},\mathcal{D}})-\mathcal{R}(\hat{F}_{P_{X|Z},\mathcal{D}^i})\rVert_\infty\leq\beta.\]
\end{definition}
The next two theorems are quoted from \citep{kadri2016operator}. 
\begin{theorem}[{\citep[Theorem 7]{kadri2016operator}}]\label{Tkadri7}
	Under assumptions 1, 2 and 3, a learning algorithm that maps a training set \(\mathcal{D}\) to the function \(\hat{F}_{P_{X|Z},\mathcal{D}}=\hat{F}_{P_{X|Z},n,\lambda}\) is \(\beta\)-stable with
	\[\beta=\frac{\tau^2\kappa_1^2}{2\lambda n}.\]
\end{theorem}
\begin{theorem}[{\citep[Theorem 8]{kadri2016operator}}]\label{Tkadri8}
	Let \(\mathcal{D}\mapsto\hat{F}_{P_{X|Z},\mathcal{D}}=\hat{F}_{P_{X|Z},n,\lambda}\) be a learning algorithm with uniform stability \(\beta\), and assume Assumption 4 is satisfied. Then, for all \(n\geq1\) and any \(0<\delta<1\), the following bound holds with probability at least \(1-\delta\) over the random draw of training samples:
	\[\tilde{\mathcal{E}}_{X|Z}(\hat{F}_{P_{X|Z},n,\lambda})\leq\frac{1}{n}\hat{\mathcal{E}}_{X|Z,n}(\hat{F}_{P_{X|Z},n,\lambda})+2\beta+(4n\beta+\kappa_2)\sqrt{\frac{\ln{\frac{1}{\delta}}}{2n}}.\]
\end{theorem}
Theorems \ref{Tkadri7} and \ref{Tkadri8} give us results about the generalisability of our learning algorithm. It remains to check whether the assumptions are satisfied. 

Assumption 2 is satisfied thanks to our assumption that point embeddings are measurable functions, and Assumption 1 is satisfied if we assume that \(k_\mathcal{Z}\) is a bounded kernel (i.e. there exists \(B_\mathcal{Z}>0\) such that \(k_\mathcal{Z}(z_1,z_2)\leq B_\mathcal{Z}\) for all \(z_1,z_2\in\mathcal{Z}\)), because
\[\lVert l_{\mathcal{X}\mathcal{Z}}(z,z)\rVert_{\text{op}}=\sup_{f\in\mathcal{H}_\mathcal{X},\lVert f\rVert_{\mathcal{H}_\mathcal{X}}=1}\lVert k_\mathcal{Z}(z,z)(f)\rVert_{\mathcal{H}_\mathcal{X}}\leq B_\mathcal{Z}.\]
In \cite{kadri2016operator}, a general loss function is used rather than the squared loss, and it is noted that Assumption 3 is in general \textit{not} satisfied with the squared loss, which is what we use in our context. However, this issue can be addressed if we restrict the output space to a bounded subset. In fact, the only elements in \(\mathcal{H}_\mathcal{X}\) that appear as the output samples in our case are \(k_\mathcal{X}(x,\cdot)\) for \(x\in\mathcal{X}\), so if we place the assumption that \(k_\mathcal{X}\) is a bounded kernel (i.e. there exists \(B_\mathcal{X}>0\) such that \(k_\mathcal{X}(x_1,x_2)\leq B_\mathcal{X}\) for all \(x_1,x_2\in\mathcal{X}\)), then by the reproducing property, 
\[\lVert k_\mathcal{X}(x,\cdot)\rVert_{\mathcal{H}_\mathcal{X}}=\sqrt{k_\mathcal{X}(x,x)}\leq\sqrt{B_\mathcal{X}}.\]
So it is no problem, in our case, to place this boundedness assumption. \citep[Appendix D]{kadri2016operator} tells us that Assumption 1 with this boundedness assumption implies Assumption 4 with
\[\kappa_2=B_\mathcal{X}\left(1+\frac{\kappa_1}{\sqrt{\lambda}}\right)^2,\]
while \citep[Lemma 2]{kadri2016operator} provides us with a condition which can replace Assumption 3 in Theorem \ref{Tkadri7}, giving us the uniform stability of our algorithm with
\[\beta=\frac{2\kappa_1^2B_\mathcal{X}\left(1+\frac{\kappa_1}{\sqrt{\lambda}}\right)^2}{\lambda n}.\]
Then the result of Theorem \ref{Tkadri8} holds with this new \(\beta\). 

\section{Proofs}\label{Sproofs}
\begin{customproof}{Lemma 2.1}
	For each \(f\in\mathcal{H}_\mathcal{X}\), \(\int_\mathcal{X}f(x)dP_X(x)=\langle f,\mu_{P_X}\rangle_{\mathcal{H}_\mathcal{X}}\). 
\end{customproof}
\begin{proof}
	Let \(L_P\) be a functional on \(\mathcal{H}\) defined by \(L_P(f)\vcentcolon=\int_\mathcal{X}f(x)dP(x)\). Then \(L_P\) is clearly linear, and moreover, 
	\begin{alignat*}{3}
	\lvert L_P(f)\rvert&=\left\lvert\int_\mathcal{X}f(x)dP(x)\right\rvert&&\\
	&=\left\lvert\int_\mathcal{X}\langle f,k(x,\cdot)\rangle_{\mathcal{H}}dP(x)\right\rvert&&\text{by the reproducing property}\\
	&\leq\int_\mathcal{X}\lvert\langle f,k(x,\cdot)\rangle_{\mathcal{H}}\rvert dP(x)&&\text{by Jensen's inequality}\\
	&\leq\lVert f\rVert_\mathcal{H}\int_\mathcal{X}\lVert k(x,\cdot)\rVert_\mathcal{H}dP(x)\qquad&&\text{by Cauchy-Schwarz inequalty}.
	\end{alignat*}
	Since the map \(x\mapsto k(x,\cdot)\) is Bochner \(P\)-integrable, \(L_P\) is bounded, i.e. \(L_P\in\mathcal{H}^*\). So by the Riesz Representation Theorem, there exists a unique \(h\in\mathcal{H}\) such that \(L_P(f)=\langle f,h\rangle_{\mathcal{H}}\) for all \(f\in\mathcal{H}\). 
	
	Choose \(f(\cdot)=k(x,\cdot)\) for some \(x\in\mathcal{X}\). Then
	\begin{alignat*}{2}
	h(x)&=\langle k(x,\cdot),h\rangle_{\mathcal{H}}\\
	&=L_P(k(x,\cdot))\\
	&=\int_\mathcal{X}k(x',x)dP(x'),
	\end{alignat*}
	which means \(h(\cdot)=\int_\mathcal{X}k(x,\cdot)dP(x)=\mu_P(\cdot)\) (implicitly applying \citep[Corollary 37]{dinculeanu2000vector}).
\end{proof}
\begin{customproof}{Lemma 2.3}
	For \(f\in\mathcal{H}_\mathcal{X}\), \(g\in\mathcal{H}_\mathcal{Y}\), \(\langle f\otimes g,\mu_{P_{\mathit{XY}}}\rangle_{\mathcal{H}_\mathcal{X}\otimes\mathcal{H}_\mathcal{Y}}=\mathbb{E}_{\mathit{XY}}[f(X)g(Y)]\).
\end{customproof}
\begin{proof}
	For Bochner integrability, we see that
	\begin{alignat*}{2}
	\mathbb{E}_{XY}\left[\left\lVert k_\mathcal{X}(X,\cdot)\otimes k_\mathcal{Y}(Y,\cdot)\right\rVert_{\mathcal{H}_\mathcal{X}\otimes\mathcal{H}_\mathcal{Y}}\right]&=\mathbb{E}_{XY}\left[\sqrt{k_\mathcal{X}(X,X)}\sqrt{k_\mathcal{Y}(Y,Y)}\right]\\
	&\leq\sqrt{\mathbb{E}_X\left[k_\mathcal{X}(X,X)\right]}\sqrt{\mathbb{E}_Y\left[k_\mathcal{Y}(Y,Y)\right]},
	\end{alignat*}
	by Cauchy-Schwarz inequality. (\ref{Estrongerintegrability}) now implies that \(k_\mathcal{X}(X,\cdot)\otimes k_\mathcal{Y}(Y,\cdot)\) is Bochner \(P_{XY}\)-integrable. 
	
	Let \(L_{P_{XY}}\) be a functional on \(\mathcal{H}_\mathcal{X}\otimes\mathcal{H}_\mathcal{Y}\) defined by \(L_{P_{XY}}\left(\sum_if_i\otimes g_i\right)\vcentcolon=\mathbb{E}_{XY}\left[\sum_if_i(X)g_i(Y)\right]\). Then \(L_{P_{XY}}\) is clearly linear, and moreover,
	\begin{alignat*}{3}
	\lvert&L_{P_{XY}}(\sum_if_i\otimes g_i)\rvert=\lvert\mathbb{E}_{XY}[\sum_if_i(X)g_i(Y)]\rvert\\
	&\leq\mathbb{E}_{XY}[\lvert\sum_if_i(X)g_i(Y)\rvert]&&\text{by Jensen's inequality}\\
	&=\mathbb{E}_{XY}[\lvert\langle\sum_if_i\otimes g_i,k_\mathcal{X}(X,\cdot)\otimes k_\mathcal{Y}(Y,\cdot)\rangle_{\mathcal{H}_\mathcal{X}\otimes\mathcal{H}_\mathcal{Y}}\rvert]&&\text{by the reproducing property}\\
	&\leq\lVert\sum_if_i\otimes g_i\rVert_{\mathcal{H}_\mathcal{X}\otimes\mathcal{H}_\mathcal{Y}}\mathbb{E}_{XY}\left[\left\lVert k_\mathcal{X}(X,\cdot)\otimes k_\mathcal{Y}(Y,\cdot)\right\rVert_{\mathcal{H}_\mathcal{X}\otimes\mathcal{H}_\mathcal{Y}}\right]\quad&&\text{by Cauchy-Schwarz inequality.}
	\end{alignat*}
	Hence, by Bochner integrability shown above, \(L_{P_{XY}}\in(\mathcal{H}_\mathcal{X}\otimes\mathcal{H}_\mathcal{Y})^*\). So by the Riesz Representation Theorem, there exists \(h\in\mathcal{H}_\mathcal{X}\otimes\mathcal{H}_\mathcal{Y}\) such that \(L_{P_{XY}}(\sum_if_i\otimes g_i)=\langle\sum_if_i\otimes g_i,h\rangle_{\mathcal{H}_\mathcal{X}\otimes\mathcal{H}_\mathcal{Y}}\) for all \(\sum_if_i\otimes g_i\in\mathcal{H}_\mathcal{X}\otimes\mathcal{H}_\mathcal{Y}\). 
	
	Choose \(k_\mathcal{X}(x,\cdot)\otimes k_\mathcal{Y}(y,\cdot)\in\mathcal{H}_\mathcal{X}\otimes\mathcal{H}_\mathcal{Y}\) for some \(x\in\mathcal{X}\) and \(y\in\mathcal{Y}\). Then
	\begin{alignat*}{2}
	h(x,y)&=\langle k_\mathcal{X}(x,\cdot)\otimes k_\mathcal{Y}(y,\cdot),h\rangle_{\mathcal{H}_\mathcal{X}\otimes\mathcal{H}_\mathcal{Y}}\qquad\text{by the reproducing property}\\
	&=L_{P_{XY}}(k_\mathcal{X}(x,\cdot)\otimes k_\mathcal{Y}(y,\cdot))\\
	&=\mathbb{E}_{XY}\left[k_\mathcal{X}(x,X)\otimes k_\mathcal{Y}(y,Y)\right]\\
	&=\mu_{P_{XY}}(x,y),
	\end{alignat*}
	as required. 
\end{proof}
\begin{lemma}\label{Lisomorphism}
	Let \(\{\varphi_i\}_{i=1}^\infty\) and \(\{\psi_j\}_{j=1}^\infty\) be orthonormal bases of \(\mathcal{H}_\mathcal{X}\) and \(\mathcal{H}_\mathcal{Y}\) respectively (note that they are countable, since the RKHSs are separable). Then the map
	\begin{alignat*}{2}
	\Phi:&\mathcal{H}_\mathcal{X}\otimes\mathcal{H}_\mathcal{Y}\rightarrow\text{HS}(\mathcal{H}_\mathcal{X},\mathcal{H}_\mathcal{Y})\\
	&\sum_{i=1,j=1}^\infty c_{i,j}(\varphi_i\otimes\psi_j)\mapsto[h\mapsto\sum_{i=1,j=1}^\infty c_{i,j}\langle h,\varphi_i\rangle_{\mathcal{H}_\mathcal{X}}\psi_j]
	\end{alignat*}
	is an isometric isomorphism.
\end{lemma}
\begin{proof}
	\(\Phi\) is clearly linear. We first show isometry:
	\begin{alignat*}{3}
	\left\lVert\Phi(\sum_{i=1,j=1}^\infty c_{i,j}(\varphi_i\otimes\psi_j))\right\rVert_{\text{HS}}^2&=\left\lVert\sum_{i=1,j=1}^\infty c_{i,j}\langle\cdot,\varphi_i\rangle_{\mathcal{H}_\mathcal{X}}\psi_j\right\rVert^2_{\text{HS}}\\
	&=\sum_{k=1}^\infty\left\lVert\sum_{i=1,j=1}^\infty c_{i,j}\langle\varphi_k,\varphi_i\rangle_{\mathcal{H}_\mathcal{X}}\psi_j\right\rVert^2_{\mathcal{H}_\mathcal{Y}}\qquad&&\text{by definition}\\
	&=\sum_{i=1,j=1}^\infty c_{i,j}^2&&\text{by orthonormality}\\
	&=\left\lVert\sum_{i=1,j=1}^\infty c_{i,j}(\varphi_i\otimes\psi_j)\right\rVert^2_{\mathcal{H}_\mathcal{X}\otimes\mathcal{H}_\mathcal{Y}}&&\text{by orthonormality},
	\end{alignat*}
	as required. It remains to show surjectivity. 
	
	Take an element \(T\in\text{HS}(\mathcal{H}_\mathcal{X},\mathcal{H}_\mathcal{Y})\). Then \(T\) is completely determined by \(\{T\varphi_i\}_{i=1}^\infty\). For each \(i\), suppose \(T\varphi_i=\sum_{j=1}^\infty d^i_j\psi_j\), with \(d^i_j\in\mathbb{R}\) for all \(i\) and \(j\). Then 
	\begin{alignat*}{3}
	\Phi\left(\sum_{i'=1,j=1}^\infty d^{i'}_j(\varphi_{i'}\otimes\psi_j)\right)&=\left[\varphi_i\mapsto\sum_{i'=1,j=1}^\infty\langle d^{i'}_j\varphi_{i'},\varphi_i\rangle_{\mathcal{H}_\mathcal{X}}\psi_j\right]\\
	&=\left[\varphi_i\mapsto\sum_{j=1}^\infty d^i_j\psi_j\right]&&\text{by orthonormality}\\
	&=T.
	\end{alignat*}
	So \(\Phi\) is surjective, and hence an isometric isomorphism. 
\end{proof}
Before we prove Theorem \ref{Tregularversion}, we state the following definition and theorems related to measurable functions for Banach-space valued functions. 
\begin{definition}[{\citep[p.4, Definition 5]{dinculeanu2000vector}}]\label{Dmeasurablefunctions}
	A function \(H:\Omega\rightarrow\mathcal{H}\) is called an \(\mathcal{F}\)-simple function if it has the form \(H=\sum^n_{i=1}h_i\mathbf{1}_{B_i}\) for some \(h_i\in\mathcal{H}\) and \(B_i\in\mathcal{F}\). 
	
	A function \(H:\Omega\rightarrow\mathcal{H}\) is said to be \(\mathcal{F}\)-measurable if there is a sequence \((H_n)\) of \(\mathcal{H}\)-valued, \(\mathcal{F}\)-simple functions such that \(H_n\rightarrow H\) pointwise. 
\end{definition}
\begin{theorem}[{\citep[p.4, Theorem 6]{dinculeanu2000vector}}]\label{Tdinculeanu6}
	If \(H:\Omega\rightarrow\mathcal{H}\) is \(\mathcal{F}\)-measurable, then there is a sequence \((H_n)\) of \(\mathcal{H}\)-valued, \(\mathcal{F}\)-simple functions such that \(H_n\rightarrow H\) pointwise and \(\lvert H_n\rvert\leq\lvert H\rvert\) for every \(n\). 
\end{theorem}
\begin{theorem}[{\citep[p.19, Theorem 48]{dinculeanu2000vector}}, Lebesgue Convergence Theorem]\label{Tdinculeanu48}
	Let \((H_n)\) be a sequence in \(L^1_\mathcal{H}(P)\), \(H:\Omega\rightarrow\mathcal{H}\) a \(P\)-measurable function, and \(g\in L^1_+(P)\) such that \(H_n\rightarrow H\) \(P\)-almost everywhere and \(\lvert H_n\rvert\leq g\), \(P\)-almost everywhere, for each \(n\). Then \(H\in L^1_\mathcal{H}(P)\) and \(H_n\rightarrow H\) in \(L^1_\mathcal{H}(P)\), i.e. \(\int_\Omega H_ndP\rightarrow\int_\Omega HdP\). 
\end{theorem}
\begin{customproof}{Theorem 2.9}
	Suppose that \(P(\cdot\mid\mathcal{E})\) admits a regular version \(Q\). Then \(QH:\Omega\rightarrow\mathcal{H}\) with \(\omega\mapsto Q_\omega H=\int_\Omega H(\omega')Q_\omega(d\omega')\)
	is a version of \(\mathbb{E}[H\mid\mathcal{E}]\) for every Bochner \(P\)-integrable \(H\). 
\end{customproof}
\begin{proof}
	Suppose \(H\) is Bochner \(P\)-integrable. Since \(Q\) is a regular version of \(P(\cdot\mid\mathcal{E})\), it is a probability transition kernel from \((\Omega,\mathcal{E})\) to \((\Omega,\mathcal{F})\). 
	
	We first show that \(QH\) is measurable with respect to \(\mathcal{E}\). The map \(Q:\Omega\rightarrow\mathcal{H}\) is well-defined, since, for each \(\omega\in\Omega\), \(Q_\omega H\) is the Bochner-integral of \(H\) with respect to the measure \(B\rightarrow Q_\omega(B)\). Since \(H\) is \(\mathcal{F}\)-measurable, by Theorem \ref{Tdinculeanu6}, there is a sequence \((H_n)\) of \(\mathcal{H}\)-valued, \(\mathcal{F}\)-simple functions such that \(H_n\rightarrow H\) pointwise. Then for each \(\omega\in\Omega\), \(Q_\omega H=\lim_{n\rightarrow\infty}Q_\omega H_n\) by Theorem \ref{Tdinculeanu48}. But for each \(n\), we can write \(H_n=\sum^m_{j=1}h_j\mathbf{1}_{B_j}\) for some \(h_j\in\mathcal{H}\) and \(B_j\in\mathcal{F}\), and so \(Q_\omega H_n=\sum^m_{j=1}h_jQ_\omega(B_j)\). For each \(B_j\) the map \(\omega\mapsto Q_\omega(B_j)\) is \(\mathcal{E}\)-measurable (by the definition of transition probability kernel, Definition \ref{Dtransitionprobabilitykernel}), and so as a linear combination of \(\mathcal{E}\)-measurable functions, \(QH_n\) is \(\mathcal{E}\)-measurable. Hence, as a pointwise limit of \(\mathcal{E}\)-measurable functions, \(QH\) is also \(\mathcal{E}\)-measurable, by \citep[p.6, Theorem 10]{dinculeanu2000vector}. 
	
	Next, we show that, for all \(A\in\mathcal{E}\), \(\int_AHdP=\int_AQHdP\). Fix \(A\in\mathcal{E}\).	By Theorem \ref{Tdinculeanu6}, there is a sequence \((H_n)\) of \(\mathcal{H}\)-valued, \(\mathcal{F}\)-simple functions such that \(H_n\rightarrow H\) pointwise. For each \(n\), we can write \(H_n=\sum^m_{j=1}h_j\mathbf{1}_{B_j}\) for some \(h_j\in\mathcal{H}\) and \(B_j\in\mathcal{F}\), and 
	\begin{alignat*}{3}
	\int_AQH_ndP&=\int_A\sum^m_{j=1}h_jQ(B_j)dP\\
	&=\int_A\sum^m_{j=1}h_jP(B_j\mid\mathcal{E})dP\quad&&\text{since }Q\text{ is a version of }P(\cdot\mid\mathcal{E})\\
	&=\sum^m_{j=1}h_j\int_A\mathbb{E}[\mathbf{1}_{B_j}\mid\mathcal{E}]dP&&\text{by the definition of conditional probability measures}\\
	&=\int_A\sum^m_{j=1}h_j\mathbf{1}_{B_j}dP&&\text{by the definition of conditional expectations, since }A\in\mathcal{E}\\
	&=\int_AH_ndP.
	\end{alignat*}
	We have \(H_n\rightarrow H\) pointwise by assertion, and as before, \(QH_n\rightarrow QH\) pointwise. Hence,
	\begin{alignat*}{3}
	\int_AQHdP&=\lim_{n\rightarrow\infty}\int_AQH_ndP\qquad&&\text{by Theorem \ref{Tdinculeanu48}}\\
	&=\lim_{n\rightarrow\infty}\int_AH_ndP&&\text{by above}\\
	&=\int_AHdP&&\text{by Theorem \ref{Tdinculeanu48}}.
	\end{alignat*}
	Hence, by the definition of the conditional expectation, \(QH\) is a version of \(\mathbb{E}[H\mid\mathcal{E}]\). 
\end{proof}
\begin{customproof}{Lemma 3.2}
	For any \(f\in\mathcal{H}_\mathcal{X}\), \(\mathbb{E}_{X|Z}[f(X)\mid Z]=\langle f,\mu_{P_{X|Z}}\rangle_{\mathcal{H}_\mathcal{X}}\) almost surely.
\end{customproof}
\begin{proof}
	The left-hand side is the conditional expectation of the real-valued random variable \(f(X)\) given \(Z\). We need to check that the right-hand side is also that. Note that \(\langle f,\mu_{P_{X|Z}}\rangle_{\mathcal{H}_\mathcal{X}}\) is clearly \(Z\)-measurable, and \(P\)-integrable (by the Cauchy-Schwarz inequality and the integrability condition (\ref{Eintegrability})). Take any \(A\in\sigma(Z)\). Then
	\begin{alignat*}{3}
	\int_A\langle f,\mu_{P_{X|Z}}\rangle_{\mathcal{H}_\mathcal{X}}dP&=\int_A\left\langle f,\mathbb{E}_{X|Z}[k_\mathcal{X}(\cdot,X)\mid Z]\right\rangle_{\mathcal{H}_\mathcal{X}}dP\quad&&\text{by definition}\\
	&=\left\langle f,\int_A\mathbb{E}_{X|Z}[k_\mathcal{X}(\cdot,X)\mid Z]dP\right\rangle_{\mathcal{H}_\mathcal{X}}&&(+)\\
	&=\left\langle f,\int_Ak_\mathcal{X}(\cdot,X)dP\right\rangle_{\mathcal{H}_\mathcal{X}}&&\text{see Definition \ref{Dconditionalexpectation}}\\
	&=\int_A\langle f,k_\mathcal{X}(\cdot,X)\rangle_{\mathcal{H}_\mathcal{X}}dP&&(+)\\
	&=\int_Af(X)dP&&\text{by the reproducing property.}
	\end{alignat*}
	Here, in \((+)\), we used the fact that the order of a continuous linear operator and Bochner integration can be interchanged \citep[p.30, Theorem 36]{dinculeanu2000vector}. Hence \(\langle f,\mu_{P_{X|Z}}\rangle_{\mathcal{H}_\mathcal{X}}\) is a version of the conditional expectation \(\mathbb{E}_{X|Z}[f(X)\mid Z]\). 
\end{proof}
\begin{customproof}{Lemma 3.3}
	For any pair \(f\in\mathcal{H}_\mathcal{X}\) and \(g\in\mathcal{H}_\mathcal{Y}\), \(\mathbb{E}_{\mathit{XY}|Z}[f(X)g(Y)\mid Z]=\langle f\otimes g,\mu_{P_{\mathit{XY}|Z}}\rangle_{\mathcal{H}_\mathcal{X}\otimes\mathcal{H}_\mathcal{Y}}\) almost surely. 
\end{customproof}
\begin{proof}
	The left-hand side is the conditional expectation of the real-valued random variable \(f(X)g(Y)\) given \(Z\). We need to check that the right-hand side is also that. Note that \(\langle f\otimes g,\mu_{P_{XY|Z}}\rangle_{\mathcal{H}_\mathcal{X}\otimes\mathcal{H}_\mathcal{Y}}\) is clearly \(Z\)-measurable, and \(P\)-integrable (by the Cauchy-Schwarz inequality and the integrability condition (\ref{Estrongerintegrability})). Take any \(A\in\sigma(Z)\). Then
	\begin{alignat*}{2}
	\int_A\langle f\otimes g,\mu_{P_{XY|Z}}\rangle_{\mathcal{H}_\mathcal{X}\otimes\mathcal{H}_\mathcal{Y}}dP&=\int_A\left\langle f\otimes g,\mathbb{E}_{XY|Z}[k_\mathcal{X}(\cdot,X)\otimes k_\mathcal{Y}(\cdot,Y)\mid Z]\right\rangle_{\mathcal{H}_\mathcal{X}\otimes\mathcal{H}_\mathcal{Y}}dP\\
	&=\left\langle f\otimes g,\int_A\mathbb{E}_{XY|Z}[k_\mathcal{X}(\cdot,X)\otimes k_\mathcal{Y}(\cdot,Y)\mid Z]dP\right\rangle_{\mathcal{H}_\mathcal{X}\otimes\mathcal{H}_\mathcal{Y}}\\
	&=\left\langle f\otimes g,\int_Ak_\mathcal{X}(\cdot,X)\otimes k_\mathcal{Y}(\cdot,Y)dP\right\rangle_{\mathcal{H}_\mathcal{X}\otimes\mathcal{H}_\mathcal{Y}}\\
	&=\int_A\langle f\otimes g,k_\mathcal{X}(\cdot,X)\otimes k_\mathcal{Y}(\cdot,Y)\rangle_{\mathcal{H}_\mathcal{X}\otimes\mathcal{H}_\mathcal{Y}}dP\\
	&=\int_Af(X)g(Y)dP.
	\end{alignat*}
	So \(\langle f\otimes g,\mu_{P_{XY|Z}}\rangle_{\mathcal{H}_\mathcal{X}\otimes\mathcal{H}_\mathcal{Y}}\) is a version of the conditional expectation \(\mathbb{E}_{XY|Z}[f(X)g(Y)\mid Z]\). 
\end{proof}
\begin{customproof}{Theorem 4.1}
	Assume that \(\mathcal{H}_\mathcal{X}\) is separable, and denote its Borel \(\sigma\)-algebra by \(\mathcal{B}(\mathcal{H}_\mathcal{X})\). Then we can write
	\[\mu_{P_{X|Z}}=F_{P_{X|Z}}\circ Z,\]
	where \(F_{P_{X|Z}}:\mathcal{Z}\rightarrow\mathcal{H}_\mathcal{X}\) is some deterministic function, measurable with respect to \(\mathfrak{Z}\) and \(\mathcal{B}(\mathcal{H}_\mathcal{X})\). 
\end{customproof}
\begin{proof}
	Let \(\Image(Z)\subseteq\mathcal{Z}\) be the image of \(Z:\Omega\rightarrow\mathcal{Z}\), and let \(\tilde{\mathfrak{Z}}\) denote the \(\sigma\)-algebra on \(\Image(Z)\) defined by \(\tilde{\mathfrak{Z}}=\{A\cap\Image(Z):A\in\mathfrak{Z}\}\) (see \citep[page 5, 1.15]{cinlar2011probability}). We will first construct a function \(\tilde{F}:\Image(Z)\rightarrow\mathcal{H}_\mathcal{X}\), measurable with respect to \(\tilde{\mathfrak{Z}}\) and \(\mathcal{B}(\mathcal{H}_\mathcal{X})\), such that \(\mu_{P_{X|Z}}=\tilde{F}\circ Z\). 
	
	For a given \(z\in\Image(Z)\subseteq\mathcal{Z}\), we have \(Z^{-1}(z)\subseteq\Omega\). Suppose for contradiction that there are two distinct elements \(\omega_1,\omega_2\in Z^{-1}(z)\) such that \(\mu_{P_{X|Z}}(\omega_1)\neq\mu_{P_{X|Z}}(\omega_2)\). Since \(\mathcal{H}_\mathcal{X}\) is Hausdorff, there are disjoint open neighbourhoods \(N_1\) and \(N_2\) of \(\mu_{P_{X|Z}}(\omega_1)\) and \(\mu_{P_{X|Z}}(\omega_2)\) respectively. By definition of a Borel \(\sigma\)-algebra, we have \(N_1,N_2\in\mathcal{B}(\mathcal{H}_\mathcal{X})\), and since \(\mu_{P_{X|Z}}\) is \(\sigma(Z)\)-measurable, 
	\begin{equation}\label{Einverseimages}
	\mu_{P_{X|Z}}^{-1}(N_1),\mu_{P_{X|Z}}^{-1}(N_2)\in\sigma(Z).
	\end{equation}
	Furthermore, \(\mu_{P_{X|Z}}^{-1}(N_1)\) and \(\mu_{P_{X|Z}}^{-1}(N_2)\) are neighbourhoods of \(\omega_1\) and \(\omega_2\) respectively, and are disjoint.
	\begin{enumerate}[(i)]
		\item For any \(B\in\tilde{\mathfrak{Z}}\) with \(z\in B\), since \(Z(\omega_1)=z=Z(\omega_2)\), we have \(\omega_1,\omega_2\in Z^{-1}(B)\). So \(Z^{-1}(B)\neq\mu_{P_{X|Z}}^{-1}(N_1)\) and \(Z^{-1}(B)\neq\mu_{P_{X|Z}}^{-1}(N_2)\), as \(\omega_2\notin\mu_{P_{X|Z}}^{-1}(N_1)\) and \(\omega_1\notin\mu_{P_{X|Z}}^{-1}(N_2)\). 
		\item For any \(B\in\tilde{\mathfrak{Z}}\) with \(z\notin B\), we have \(\omega_1\notin Z^{-1}(B)\) and \(\omega_2\notin Z^{-1}(B)\). So \(Z^{-1}(B)\neq\mu_{P_{X|Z}}^{-1}(N_1)\) and \(Z^{-1}(B)\neq\mu_{P_{X|Z}}^{-1}(N_2)\). 
	\end{enumerate}
	Since \(\sigma(Z)=\{Z^{-1}(B)\mid B\in\tilde{\mathfrak{Z}}\}\) (see \citep{cinlar2011probability}, page 11, Exercise 2.20), we can't have \(\mu_{P_{X|Z}}^{-1}(N_1)\in\sigma(Z)\) nor \(\mu_{P_{X|Z}}^{-1}(N_2)\in\sigma(Z)\). This is a contradiction to (\ref{Einverseimages}). We therefore conclude that, for any \(z\in\mathcal{Z}\), if \(Z(\omega_1)=z=Z(\omega_2)\) for distinct \(\omega_1,\omega_2\in\Omega\), then \(\mu_{P_{X|Z}}(\omega_1)=\mu_{P_{X|Z}}(\omega_2)\). 
	
	We define \(\tilde{F}(z)\) to be the unique value of \(\mu_{P_{X|Z}}(\omega)\) for all \(\omega\in Z^{-1}(z)\). Then for any \(\omega\in\Omega\), \(\mu_{P_{X|Z}}(\omega)=\tilde{F}(Z(\omega))\) by construction. It remains to check that \(\tilde{F}\) is measurable with respect to \(\tilde{\mathfrak{Z}}\) and \(\mathcal{B}(\mathcal{H}_\mathcal{X})\). 
	
	Take any \(N\in\mathcal{B}(\mathcal{H}_\mathcal{X})\). Since \(\mu_{P_{X|Z}}\) is \(\sigma(Z)\)-measurable, \(\mu_{P_{X|Z}}^{-1}(N)=Z^{-1}(\tilde{F}^{-1}(N))\in\sigma(Z)\). Since \(\sigma(Z)=\{Z^{-1}(B)\mid B\in\tilde{\mathfrak{Z}}\}\), we have \(Z^{-1}(\tilde{F}^{-1}(N))=Z^{-1}(C)\) for some \(C\in\tilde{\mathfrak{Z}}\). Since the mapping \(Z:\Omega\rightarrow\Image(Z)\) is surjective, \(\tilde{F}^{-1}(N)=C\). Hence \(\tilde{F}^{-1}(N)\in\tilde{\mathfrak{Z}}\), and so \(\tilde{F}\) is measurable with respect to \(\tilde{\mathfrak{Z}}\) and \(\mathcal{B}(\mathcal{H}_\mathcal{X})\). 
	
	Finally, we can extend \(\tilde{F}:\Image(Z)\rightarrow\mathcal{H}_\mathcal{X}\) to \(F:\mathcal{Z}\rightarrow\mathcal{H}_\mathcal{X}\) by \citep[page 128, Corollary 4.2.7]{dudley2018real} (note that \(\mathcal{H}_\mathcal{X}\) is a complete metric space, and assumed to be separable in this theorem). 
\end{proof}
\begin{customproof}{Theorem 4.2}
	\(F_{P_{X|Z}}\in L^2(\mathcal{Z},P_Z;\mathcal{H}_\mathcal{X})\) minimises both \(\tilde{\mathcal{E}}_{X|Z}\) and \(\mathcal{E}_{X|Z}\), i.e.
	\[F_{P_{X|Z}}=\argmin_{F\in L^2(\mathcal{Z},P_Z;\mathcal{H}_\mathcal{X})}\mathcal{E}_{X|Z}(F)=\argmin_{F\in L^2(\mathcal{Z},P_Z;\mathcal{H}_\mathcal{X})}\tilde{\mathcal{E}}_{X|Z}(F).\]
	Moreover, it is almost surely unique, i.e. it is almost surely equal to any other minimiser of the objective functionals. 
\end{customproof}
\begin{proof}
	Recall that we have 
	\[\mathcal{E}_{X|Z}(F)\vcentcolon=\mathbb{E}_Z\left[\lVert F_{P_{X|Z}}(Z)-F(Z)\rVert_{\mathcal{H}_\mathcal{X}}^2\right].\]
	So clearly, \(\mathcal{E}_{X|Z}(F_{P_{X|Z}})=0\), meaning \(F_{P_{X|Z}}\) minimises \(\mathcal{E}_{X|Z}\) in \(L^2(\mathcal{Z},P_Z;\mathcal{H}_\mathcal{X})\). So it only remains to show that \(\tilde{\mathcal{E}}_{X|Z}\) is minimised in \(L^2(\mathcal{Z},P_Z;\mathcal{H}_\mathcal{X})\) by \(F_{P_{X|Z}}\). 
	
	Let \(F\) be any element in \(L^2(\mathcal{Z},P_Z;\mathcal{H}_\mathcal{X})\). Then we have
	\begin{equation}\label{Egrunewalder3.1process}
	\begin{split}
	\tilde{\mathcal{E}}_{X|Z}(F)-\tilde{\mathcal{E}}_{X|Z}(F_{P_{X|Z}})&=\mathbb{E}_{X,Z}[\lVert k_\mathcal{X}(X,\cdot)-F(Z)\rVert_{\mathcal{H}_\mathcal{X}}^2]-\mathbb{E}_{X,Z}[\lVert k_\mathcal{X}(X,\cdot)-F_{P_{X|Z}}(Z)\rVert_{\mathcal{H}_\mathcal{X}}^2]\\
	&=\mathbb{E}_Z[\lVert F(Z)\rVert_{\mathcal{H}_\mathcal{X}}^2]-2\mathbb{E}_{X,Z}[\langle k_\mathcal{X}(X,\cdot),F(Z)\rangle_{\mathcal{H}_\mathcal{X}}]\\
	&\quad+2\mathbb{E}_{X,Z}\left[\langle k_\mathcal{X}(X,\cdot),F_{P_{X|Z}}(Z)\rangle_{\mathcal{H}_\mathcal{X}}\right]-\mathbb{E}_Z\left[\lVert F_{P_{X|Z}}(Z)\rVert_{\mathcal{H}_\mathcal{X}}^2\right].
	\end{split}
	\end{equation}
	Here, 
	\begin{alignat*}{3}
	\mathbb{E}_{X,Z}\left[\langle k_\mathcal{X}(X,\cdot),F(Z)\rangle_{\mathcal{H}_\mathcal{X}}\right]&=\mathbb{E}_Z\left[\mathbb{E}_{X|Z}\left[F(Z)(X)\mid Z\right]\right]&&\text{by the reproducing property}\\
	&=\mathbb{E}_Z\left[\langle F(Z),\mu_{P_{X|Z}}\rangle_{\mathcal{H}_\mathcal{X}}\right]&&\text{by Lemma \ref{Lconditionalinterchange}}\\
	&=\mathbb{E}_Z\left[\langle F(Z),F_{P_{X|Z}}(Z)\rangle_{\mathcal{H}_\mathcal{X}}\right]\qquad&&\text{since }\mu_{P_{X|Z}}=F_{P_{X|Z}}\circ Z
	\end{alignat*}
	and similarly, 
	\begin{alignat*}{3}
	\mathbb{E}_{X,Z}[\langle k_\mathcal{X}(X,\cdot),F_{P_{X|Z}}(Z)\rangle_{\mathcal{H}_\mathcal{X}}]&=\mathbb{E}_Z[\mathbb{E}_{X|Z}[F_{P_{X|Z}}(Z)(X)\mid Z]]\qquad&&\text{by the reproducing property}\\
	&=\mathbb{E}_Z\left[\langle F_{P_{X|Z}}(Z),F_{P_{X|Z}}(Z)\rangle_{\mathcal{H}_\mathcal{X}}\right]&&\text{by Lemma \ref{Lconditionalinterchange}}\\
	&=\mathbb{E}_Z\left[\lVert F_{P_{X|Z}}(Z)\rVert_{\mathcal{H}_\mathcal{X}}^2\right].
	\end{alignat*}
	Substituting these expressions back into (\ref{Egrunewalder3.1process}), we have
	\begin{alignat*}{2}
	\tilde{\mathcal{E}}_{X|Z}(F)&-\tilde{\mathcal{E}}_{X|Z}(F_{P_{X|Z}})\\
	&=\mathbb{E}_Z[\lVert F(Z)\rVert_{\mathcal{H}_\mathcal{X}}^2]-2\mathbb{E}_Z[\langle F(Z),F_{P_{X|Z}}(Z)\rangle_{\mathcal{H}_\mathcal{X}}]+\mathbb{E}_Z[\lVert F_{P_{X|Z}}(Z)\rVert_{\mathcal{H}_\mathcal{X}}^2]\\
	&=\mathbb{E}_Z[\lVert F(Z)-F_{P_{X|Z}}(Z)\rVert_{\mathcal{H}_\mathcal{X}}^2]\\
	&\geq0.
	\end{alignat*}
	Hence, \(F_{P_{X|Z}}\) minimises \(\tilde{\mathcal{E}}_{X|Z}\) in \(L^2(\mathcal{Z},P_Z;\mathcal{H}_\mathcal{X})\). The minimiser is further more \(P_Z\)-almost surely unique; indeed, if \(F'\in L^2(\mathcal{Z},P_Z;\mathcal{H}_\mathcal{X})\) is another minimiser of \(\tilde{\mathcal{E}}_{X|Z}\), then the calculation in (\ref{Egrunewalder3.1process}) shows that
	\[\mathbb{E}_Z\left[\lVert F_{P_{X|Z}}(Z)-F'(Z)\rVert_{\mathcal{H}_\mathcal{X}}^2\right]=0,\]
	which immediately implies that \(\lVert F_{P_{X|Z}}(Z)-F'(Z)\rVert_{\mathcal{H}_\mathcal{X}}=0\) \(P_Z\)-almost surely, which in turn implies that \(F_{P_{X|Z}}=F'\) \(P_Z\)-almost surely. 
\end{proof}
\begin{customproof}{Theorem 4.4}
	Suppose that \(k_\mathcal{X}\) and \(k_\mathcal{Z}\) are bounded kernels, i.e. there exist \(B_\mathcal{Z},B_\mathcal{X}>0\) such that \(\sup_{z\in\mathcal{Z}}k_\mathcal{Z}(z,z)\leq B_\mathcal{Z}\) and \(\sup_{x\in\mathcal{X}}k_\mathcal{X}(x,x)\leq B_\mathcal{X}\), and that the operator-valued kernel \(l_{\mathcal{X}\mathcal{Z}}\) is \(\mathcal{C}_0\)-universal. Let the regularisation parameter \(\lambda_n\) decay to 0 at a slower rate than \(\mathcal{O}(n^{-1/2})\). Then our learning algorithm that produces \(\hat{F}_{P_{X|Z},n,\lambda_n}\) is universally consistent (in the surrogate loss \(\tilde{\mathcal{E}}_{X|Z}\)), i.e. for any joint distribution \(P_{XZ}\) and constants \(\epsilon>0\) and \(\delta>0\),
	\[P_{\mathit{XZ}}(\tilde{\mathcal{E}}_{X|Z}(\hat{F}_{P_{X|Z},n,\lambda_n})-\tilde{\mathcal{E}}_{X|Z}(F_{P_{X|Z}})>\epsilon)<\delta\]
	for large enough \(n\). 
\end{customproof}
\begin{proof}
	Follows immediately from \citep[Theorem 2.3]{park2020regularised}. 

\end{proof}
\begin{customproof}{Theorem 4.5}
	In addition to the setting in Theorem \ref{Tconsistency}, assume that \(F_{P_{X|Z}}\in\mathcal{G}_{\mathcal{X}\mathcal{Z}}\). Let the regularisation parameter \(\lambda_n\) decay to 0 with rate \(\mathcal{O}(n^{-1/4})\). Then \(\tilde{\mathcal{E}}_{X|Z}(\hat{F}_{P_{X|Z},n,\lambda_n})-\tilde{\mathcal{E}}_{X|Z}(F_{P_{X|Z}})=\mathcal{O}_P(n^{-1/4})\). 
\end{customproof}
\begin{proof}
	Follows immediately from \citep[Theorem 2.4]{park2020regularised}. 
\end{proof}
\begin{customproof}{Theorem 5.2}
	Suppose that \(k_\mathcal{X}\) is a characteristic kernel, that \(P_Z\) and \(P_{Z'}\) are absolutely continuous with respect to each other, and that \(P(\cdot\mid Z)\) and \(P(\cdot\mid Z')\) admit regular versions. Then \(\textnormal{MCMD}_{P_{X|Z},P_{X'|Z'}}=0\) \(P_Z\)- (or \(P_{Z'}\)-)almost everywhere if and only if, for \(P_Z\)- (or \(P_{Z'}\)-)almost all \(z\in\mathcal{Z}\), \(P_{X|Z=z}(B)=P_{X'|Z'=z}(B)\) for all \(B\in\mathfrak{X}\). 
\end{customproof}
\begin{proof}
	Write \(Q\) and \(Q'\) for some regular versions of \(P(\cdot\mid Z)\) and \(P(\cdot\mid Z')\) respectively, and assume without loss of generality that the conditional distributions \(P_{X|Z}\) and \(P_{X'|Z'}\) are given by \(P_{X|Z}(\omega)(B)=Q_\omega(X\in B)\) and \(P_{X'|Z'}(\omega)(B)=Q'_\omega(X'\in B)\) for \(B\in\mathfrak{X}\). By the definition of regular versions, for each \(B\in\mathfrak{X}\), the real-valued random variables \(\omega\mapsto P_{X|Z}(\omega)(B)\) and \(\omega\mapsto P_{X'|Z'}(\omega)(B)\) are measurable with respect to \(Z\) and \(Z'\) respectively, and so there are functions \(R_B:\mathcal{Z}\rightarrow\mathbb{R}\) and \(R'_B:\mathcal{Z}\rightarrow\mathbb{R}\) such that \(P_{X|Z}(\omega)(B)=R_B(Z(\omega))\) and \(P_{X'|Z'}(\omega)(B)=R'_B(Z'(\omega))\). Moreover, for each fixed \(z\in\mathcal{Z}\), the mappings \(B\mapsto P_{X|Z}(Z^{-1}(z))(B)=R_B(z)\) and \(B\mapsto P_{X'|Z'}(Z'^{-1}(z))(B)=R'_B(z)\) are measures. We write \(R_B(z)=P_{X|Z=z}(B)\) and \(R'_B(z)=P_{X'|Z'=z}(B)\). 
	
	By Theorem \ref{Tregularversion}, there exists an event \(A_1\in\mathcal{F}\) with \(P(A_1)=1\) such that for all \(\omega\in A_1\),
	\[\mu_{P_{X|Z}}(\omega)\vcentcolon=\mathbb{E}_{X|Z}[k_\mathcal{X}(X,\cdot)\mid Z](\omega)=\int_{\Omega}k_\mathcal{X}(X(\omega'),\cdot)Q_\omega(d\omega')=\int_\mathcal{X}k_\mathcal{X}(x,\cdot)P_{X|Z}(\omega)(dx),\]
	and an event \(A_2\in\mathcal{F}\) with \(P(A_2)=1\) such that for all \(\omega\in A_2\),
	\begin{alignat*}{2}
	\mu_{P_{X'|Z'}}(\omega)\vcentcolon=\mathbb{E}_{X'|Z'}[k_\mathcal{X}(X',\cdot)\mid Z'](\omega)&=\int_{\Omega}k_\mathcal{X}(X'(\omega'),\cdot)Q_\omega(d\omega')\\
	&=\int_\mathcal{X}k_\mathcal{X}(x',\cdot)P_{X'|Z'}(\omega)(dx').
	\end{alignat*}
	
	Suppose for contradiction that there exists some \(D\in\mathfrak{Z}\) with \(P_Z(D)>0\) such that for all \(z\in D\), \(F_{P_{X|Z}}(z)\neq\int_\mathcal{X}k_\mathcal{X}(x,\cdot)R_{dx}(z)\). Then \(P(Z^{-1}(D))=P_Z(D)>0\), and hence \(P(Z^{-1}(D)\cap A_1)>0\). For all \(\omega\in Z^{-1}(D)\cap A_1\), we have \(Z(\omega)\in D\), and hence
	\[\mu_{P_{X|Z}}(\omega)=F_{P_{X|Z}}(Z(\omega))\neq\int_\mathcal{X}k_\mathcal{X}(x,\cdot)R_{dx}(Z(\omega))=\int_\mathcal{X}k_\mathcal{X}(x,\cdot)P_{X|Z}(\omega)(dx).\]
	This contradicts our assertion that \(\mu_{P_{X|Z}}(\omega)=\int_\mathcal{X}k_\mathcal{X}(x,\cdot)P_{X|Z}(\omega)(dx)\) for all \(\omega\in A_1\), hence there does not exist \(D\in\mathfrak{Z}\) with \(P_Z(D)>0\) such that for all \(z\in D\), \(F_{P_{X|Z}}(z)\neq\int_\mathcal{X}k_\mathcal{X}(x,\cdot)R_{dx}(z)\). Therefore, there must exist some \(C_1\in\mathfrak{Z}\) with \(P_Z(C_1)=1\) such that for all \(z\in C_1\), \(F_{P_{X|Z}}(z)=\int_\mathcal{X}k_\mathcal{X}(x,\cdot)R_{dx}(z)\). Similarly, there must exist some \(C_2\in\mathfrak{Z}\) with \(P_Z(C_2)=1\) such that for all \(z\in C_2\), \(F_{P_{X'|Z'}}(z)=\int_\mathcal{X}k_\mathcal{X}(x,\cdot)R'_{dx}(z)\). Since \(P_Z\) and \(P_{Z'}\) are absolutely continuous with respect to each other, we also have \(P_Z(C_2)=1=P_{Z'}(C_1)\). 
	\begin{description}
		\item[(\(\implies\))] Suppose first that \(\text{MCMD}_{P_{X|Z},P_{X'|Z'}}=\lVert F_{P_{X|Z}}-F_{P_{X'|Z'}}\rVert_{\mathcal{H}_\mathcal{X}}=0\) \(P_Z\)-almost everywhere, i.e. there exists \(C\in\mathfrak{Z}\) with \(P_Z(C)=1\) such that for all \(z\in C\), \(\lVert F_{P_{X|Z}}(z)-F_{P_{X'|Z'}}(z)\rVert_{\mathcal{H}_\mathcal{X}}=0\). Then for each \(z\in C\cap C_1\cap C_2\), 
		\begin{alignat*}{3}
		\int_\mathcal{X}k_\mathcal{X}(x,\cdot)R_{dx}(z)&=F_{P_{X|Z}}(z)&&\text{since }z\in C_1\\
		&=F_{P_{X'|Z'}}(z)&&\text{since }z\in C\\
		&=\int_\mathcal{X}k_\mathcal{X}(x,\cdot)R'_{dx}(z)\qquad&&\text{since }z\in C_2.
		\end{alignat*}
		Since the kernel \(k_\mathcal{X}\) is characteristic, this means that \(B\mapsto R_B(z)\) and \(B\mapsto R'_B(z)\) are the same probability measure on \((\mathcal{X},\mathfrak{X})\). By countable intersection, we have \(P_Z(C\cap C_1\cap C_2)=1\), so \(P_Z\)-almost everywhere,
		\[P_{X|Z=z}(B)=P_{X'|Z'=z}(B)\]
		for all \(B\in\mathfrak{X}\). 
		\item[(\(\impliedby\))] Now assume there exists \(C\in\mathfrak{Z}\) with \(P_Z(C)=1\) such that for each \(z\in C\), \(R_B(z)=R'_B(z)\) for all \(B\in\mathfrak{X}\). Then for all \(z\in C\cap C_1\cap C_2\), 
		\begin{alignat*}{3}
		&\left\lVert F_{P_{X|Z}}(z)-F_{P_{X'|Z'}}(z)\right\rVert_{\mathcal{H}_\mathcal{X}}\\
		&=\left\lVert\int_\mathcal{X}k_\mathcal{X}(x,\cdot)R_{dx}(z)-\int_\mathcal{X}k_\mathcal{X}(x,\cdot)R'_{dx}(z)\right\rVert_{\mathcal{H}_\mathcal{X}}\quad&\text{since }z\in C_1\cap C_2\\
		&=\left\lVert\int_\mathcal{X}k_\mathcal{X}(x,\cdot)R_{dx}(z)-\int_\mathcal{X}k_\mathcal{X}(x,\cdot)R_{dx}(z)\right\rVert_{\mathcal{H}_\mathcal{X}}&\text{since }z\in C\\
		&=0,
		\end{alignat*}
		and since \(P_Z(C\cap C_1\cap C_2)=1\), \(\lVert F_{P_{X|Z}}-F_{P_{X'|Z'}}\rVert_{\mathcal{H}_\mathcal{X}}=0\) \(P_Z\)-almost everywhere. 
	\end{description}
\end{proof}
\begin{customproof}{Theorem 5.4}
	Suppose \(k_\mathcal{X}\otimes k_\mathcal{Y}\) is a characteristic kernel on \(\mathcal{X}\times\mathcal{Y}\), and that \(P(\cdot\mid Z)\) admits a regular version. Then \(\textnormal{HSCIC}(X,Y\mid Z)=0\) almost surely if and only if \(X\independent Y\mid Z\). 
\end{customproof}
\begin{proof}
	Write \(Q\) for a regular version of \(P(\cdot\mid Z)\), and assume without loss of generality that the conditional distributions \(P_{X|Z}\), \(P_{Y|Z}\) and \(P_{XY|Z}\) are given by \(P_{X|Z}(\omega)(B)=Q_\omega(X\in B)\) for \(B\in\mathcal{X}\), \(P_{Y|Z}(\omega)(C)=Q_\omega(Y\in C)\) for \(C\in\mathfrak{Y}\) and \(P_{XY|Z}(\omega)(D)=Q_\omega((X,Y)\in D)\) for \(D\in\mathfrak{X}\times\mathfrak{Y}\). By Theorem \ref{Tregularversion}, there exists an event \(A_1\in\mathcal{F}\) with \(P(A_1)=1\) such that for all \(\omega\in A_1\),
	\[\mu_{P_{X|Z}}(\omega)\vcentcolon=\mathbb{E}_{X|Z}[k_\mathcal{X}(X,\cdot)\mid Z](\omega)=\int_{\Omega}k_\mathcal{X}(X(\omega'),\cdot)Q_\omega(d\omega')=\int_\mathcal{X}k_\mathcal{X}(x,\cdot)P_{X|Z}(\omega)(dx),\]
	an event \(A_2\in\mathcal{F}\) with \(P(A_2)=1\) such that for all \(\omega\in A_2\),
	\[\mu_{P_{Y|Z}}(\omega)\vcentcolon=\mathbb{E}_{Y|Z}[k_\mathcal{Y}(Y,\cdot)\mid Z](\omega)=\int_{\Omega}k_\mathcal{Y}(Y(\omega'),\cdot)Q_\omega(d\omega')=\int_\mathcal{Y}k_\mathcal{Y}(y,\cdot)P_{Y|Z}(\omega)(dy),\]
	and an event \(A_3\in\mathcal{F}\) with \(P(A_3)=1\) such that for all \(\omega\in A_3\),
	\[\mu_{P_{XY|Z}}(\omega)=\int_{\mathcal{X}\times\mathcal{Y}}k_\mathcal{X}(x,\cdot)\otimes k_\mathcal{Y}(y,\cdot)P_{XY|Z}(\omega)(d(x,y)).\]
	This means that, for each \(\omega\in A_1\), \(\mu_{P_{X|Z}}(\omega)\) is the mean embedding of \(P_{X|Z}(\omega)\), and for each \(\omega\in A_2\), \(\mu_{P_{Y|Z}}(\omega)\) is the mean embedding of \(P_{Y|Z}(\omega)\). 
	\begin{description}
		\item[(\(\implies\))] Suppose first that \(\text{HSCIC}(X,Y\mid Z)=\lVert \mu_{P_{XY|Z}}-\mu_{P_{X|Z}}\otimes\mu_{P_{Y|Z}}\rVert_{\mathcal{H}_\mathcal{X}\otimes\mathcal{H}_\mathcal{Y}}=0\) almost surely, i.e. there exists \(A\in\mathcal{F}\) with \(P(A)=1\) such that for all \(\omega\in A\), \(\lVert \mu_{P_{XY|Z}}(\omega)-\mu_{P_{X|Z}}(\omega)\otimes\mu_{P_{Y|Z}}(\omega)\rVert_{\mathcal{H}_\mathcal{X}\otimes\mathcal{H}_\mathcal{Y}}=0\). Then for each \(\omega\in A\cap A_1\cap A_2\cap A_3\), 
		\begin{alignat*}{3}
		\int_{\mathcal{X}\times\mathcal{Y}}&k_\mathcal{X}(x,\cdot)\otimes k_\mathcal{Y}(y,\cdot)P_{XY|Z}(\omega)(d(x,y))=\mu_{P_{XY|Z}}(\omega)&&\text{since }\omega\in A_3\\
		&=\mu_{P_{X|Z}}(\omega)\otimes\mu_{P_{Y|Z}}(\omega)&&\text{since }\omega\in A\\
		&=\int_\mathcal{X}k_\mathcal{X}(x,\cdot)P_{X|Z}(\omega)(dx)\otimes\int_\mathcal{Y}k_\mathcal{Y}(y,\cdot)P_{Y|Z}(\omega)(dy)\quad&&\text{since }\omega\in A_1\cap A_2\\
		&=\int_{\mathcal{X}\times\mathcal{Y}}k_\mathcal{X}(x,\cdot)\otimes k_\mathcal{Y}(y,\cdot)P_{X|Z}(\omega)P_{Y|Z}(\omega)(d(x,y))&&\text{by Fubini.}
		\end{alignat*}
		Since the kernel \(k_\mathcal{X}\otimes k_\mathcal{Y}\) is characteristic, the distributions \(P_{XY|Z}(\omega)\) and \(P_{X|Z}(\omega)P_{Y|Z}(\omega)\) on \(\mathcal{X}\times\mathcal{Y}\) are the same. By countable intersection, we have \(P(A\cap A_1\cap A_2\cap A_3)=1\), so \(P_{XY|Z}\) and \(P_{X|Z}P_{Y|Z}\) are the same almost surely, and we have \(X\independent Y\mid Z\). 
		\item[(\(\impliedby\))] Now assume \(X\independent Y\mid Z\), i.e. there exists \(A\in\mathcal{F}\) with \(P(A)=1\) such that for each \(\omega\in A\), the distributions \(P_{XY|Z}(\omega)\) and \(P_{X|Z}(\omega)P_{Y|Z}(\omega)\) are the same. Then for all \(\omega\in A\cap A_1\cap A_2\cap A_3\), 
		\begin{alignat*}{3}
		\mu_{P_{XY|Z}}(\omega)&=\int_{\mathcal{X}\times\mathcal{Y}}k_\mathcal{X}(x,\cdot)\otimes k_\mathcal{Y}(y,\cdot)P_{XY|Z}(\omega)(d(x,y))&&\text{since }\omega\in A_3\\
		&=\int_{\mathcal{X}\times\mathcal{Y}}k_\mathcal{X}(x,\cdot)\otimes k_\mathcal{Y}(y,\cdot)P_{X|Z}(\omega)(dx)P_{Y|Z}(\omega)(dy)\quad&&\text{since }\omega\in A\\
		&=\int_\mathcal{X}k_\mathcal{X}(x,\cdot)P_{X|Z}(\omega)(dx)\otimes\int_\mathcal{Y}k_\mathcal{Y}(y,\cdot)P_{Y|Z}(\omega)(dy)&&\text{by Fubini}\\
		&=\mu_{P_{X|Z}}(\omega)\otimes\mu_{P_{Y|Z}}(\omega)&&\text{since }\omega\in A_1\cap A_2.
		\end{alignat*}
		and since \(P(A\cap A_1\cap A_2\cap A_3)=1\), \(\text{HSCIC}(X,Y\mid Z)=0\) almost surely.  
	\end{description}
\end{proof}
\end{document}